\documentclass{article}
\usepackage{fullpage}
\usepackage{microtype}
\usepackage{graphicx}
\usepackage{subcaption}
\usepackage{booktabs} 
\usepackage{hyperref}
\usepackage{natbib}
\usepackage{tikz}
\usetikzlibrary{positioning, shapes, fit, backgrounds, arrows.meta, calc}

\usepackage{amsmath}
\usepackage{amssymb}
\usepackage{mathtools}
\usepackage{amsthm}

\usepackage[capitalize,noabbrev]{cleveref}
\usepackage{algorithm}
\usepackage{algorithmic}
\usepackage{dsfont}
\usepackage[mathscr]{euscript}
\usepackage{xcolor}
\usepackage{colortbl}
\usepackage{thmtools}
\usepackage{thm-restate}

\usepackage{multirow}
\usepackage{wrapfig}

\usepackage{pifont}%

\usepackage{MnSymbol}
\DeclareMathAlphabet\mathbb{U}{msb}{m}{n}
\usepackage{xpatch}

\def\Rset{\mathbb{R}}

\let\Pr\undefined

\DeclareMathOperator*{\Pr}{\mathbb{P}}

\DeclareMathOperator*{\E}{\mathbb E}
\DeclareMathOperator*{\argmax}{argmax}

\DeclareMathOperator{\Regret}{Regret}
\DeclareMathOperator*{\TV}{TV}

\DeclarePairedDelimiter{\abs}{\lvert}{\rvert} 
\DeclarePairedDelimiter{\bracket}{[}{]}
\DeclarePairedDelimiter{\curl}{\{}{\}}
\DeclarePairedDelimiter{\paren}{(}{)}
\DeclarePairedDelimiter{\norm}{\|}{\|}

\ExplSyntaxOn
\tl_const:Nn \c_my_uc_alphabet_tl { 
ABCDEFGHIJKLMNOPQRSTUVWXYZ }
\tl_const:Nn \c_my_full_alphabet_tl { ABCDEFGHIJKLMNOPQRSTUVWXYZ
  abcdefghijklmnopqrstuvwxyz }

\tl_map_inline:Nn \c_my_uc_alphabet_tl
 { \cs_gset:cpn { c#1 } { \mathcal{#1} } }

\tl_map_inline:Nn \c_my_uc_alphabet_tl
 { \cs_gset:cpn { s#1 } { \mathscr{#1} } }

\tl_map_inline:Nn \c_my_full_alphabet_tl
 {
  \str_if_eq:nnTF { #1 } { f }
    { 
    }
    {
      \cs_gset:cpn { b#1 } { \mathbf{#1} }
    }
  
  \cs_gset:cpn { sf#1 } { \mathsf{#1} }
 }
\ExplSyntaxOff

\newcommand{\h}{\widehat}
\newcommand{\ov}{\overline}
\newcommand{\wt}{\widetilde}
\newcommand{\e}{\epsilon}
\newcommand{\ignore}[1]{}

\newcommand{\KL}{\sfD_{\mathrm{KL}}}

\hypersetup{
  breaklinks   = true, %
  colorlinks   = true, %
  urlcolor     = blue, %
  linkcolor    = blue, %
  citecolor    = blue %
}

\usepackage{enumitem}

\usepackage[toc, page, header]{appendix}
\declaretheorem{theorem}
\newtheorem{lemma}[theorem]{Lemma} 
\newtheorem{proposition}[theorem]{Proposition} 

\newtheorem{corollary}[theorem]{Corollary}
\newtheorem{definition}[theorem]{Definition}

\definecolor{brightorange}{RGB}{255, 150, 0}

\begin{document}
\title{Theoretical Foundations and Effective Algorithms \\
  for Policy-Aware Simulator Learning}
\author{%
  Christoph Dann\thanks{
  Google Research.
\texttt{cdann@cdann.net} }
  \and
  Yishay Mansour\thanks{
  Tel Aviv University and Google Research.
  \texttt{mansour.yishay@gmail.com}.
  Supported in part by European Research Council (ERC) under the European
  Union’s Horizon 2020 research and innovation program (grant agreement
  No. 882396), by the Israel Science Foundation, the Yandex Initiative for
  Machine Learning at Tel Aviv University and a grant from the Tel Aviv
  University Center for AI and Data Science (TAD).
}
  \and
  Mehryar Mohri\thanks{
    Google Research (\texttt{mohri@google.com}) and
    Courant Institute of Mathematical
  Sciences.
}}
\date{}
\maketitle

\begin{abstract}
Model-based reinforcement learning (MBRL) agents typically learn world
models by minimizing predictive loss. However, powerful RL optimizers
inevitably exploit minor model inaccuracies, leading to ``simulator
exploitation'' and a reality gap where policies succeed in simulation
but fail in the real world.

We propose that the objective for learning simulators should be
\emph{strategic robustness} rather than predictive accuracy, and
formulate this as a zero-sum minimax game between a model player and
an adversarial policy player. We provide a comprehensive theoretical
analysis: (1)~an online learning guarantee showing the game is
learnable with sublinear regret bounds; (2)~a tractable critic-based
simplification bounding the global policy-value gap by the local
critic's loss; and (3)~an \emph{Error-MDP duality}, proving that
finding the worst-case policy is formally dual to a standard RL
problem where the reward is the one-step critic error. This duality
yields a provably convergent active data selection algorithm.
Experiments on continuous control tasks demonstrate that our approach
reduces prediction error in strategically important regions by
$1.5$--$2.2\times$ and enables policies trained purely in simulation
to match near-optimal real-world performance.
\end{abstract}

\section{Introduction}
\label{sec:introduction}

Model-based reinforcement learning (MBRL) represents a promising path
toward agents with high sample efficiency, strong planning
capabilities, and the ability to test policies in a safe, offline
environment \citep{sutton1990dyna, hafner2019dreamer}. In this
paradigm, an agent first learns a simulator, or ``world model''
($\h M$), from data collected in the real world ($M$). Second, the
agent trains its policy ($\pi$) entirely within this learned
simulator, using it as a fast and safe sandbox to find an optimal
strategy.

The critical flaw in this decoupled approach is the inevitable
discrepancy between the learned model and the true environment. No
matter how much data is collected or how large the model, the
simulator is never a perfect copy of reality. This discrepancy leads
to a well-known failure mode: the reinforcement learning algorithm, as
a powerful optimizer, will inevitably discover and \emph{exploit} the
tiniest inaccuracies in the simulator's physics or rules to achieve an
artificially high score \citep{chua2018pets}. This phenomenon, often
called the \emph{reality gap} or \emph{simulator exploitation},
results in policies that are brilliant in simulation but fail
catastrophically when deployed in the real world.

The standard approach to simulator learning is to minimize a one-step
predictive loss, such as the Mean Squared Error (MSE) or the
Kullback-Leibler (KL) divergence via Maximum Likelihood Estimation
(MLE). This objective, however, is fundamentally misaligned with the
agent's downstream task. MLE optimizes for \emph{average predictive
  accuracy} across all visited states, whereas the RL agent's
performance is determined by the \emph{worst-case, exploitable error}
along the specific trajectories it discovers. A model can have 99\%
predictive accuracy on average but fail on a single, crucial
transition that an adversary can learn to visit repeatedly.

Existing work has proposed several ways to mitigate this. One popular
method is to learn an ensemble of models to quantify uncertainty
\citep{chua2018pets} and train a policy that is robust to this
distribution. Another line of work, particularly in offline RL,
introduces pessimism or conservatism, penalizing the agent for actions
that deviate from the known data \citep{yu2020mopo,
  kidambi2020morel}. A third category uses robust control
formulations, but these often require a pre-defined and known
"uncertainty set" for the dynamics \citep{nilim2005robust,
  wiesemann2013robust}, which is unavailable when the model is
learned.

More recently, adversarial methods have been explored. Some works
introduce adversarial disturbances to the state
\citep{pinto2017robust} rather than the model itself. While promising, this still optimizes
a local, one-step metric of realism. It fails to address the global,
policy-level objective of \emph{strategic robustness}. The critic does
not ask, \emph{Is this transition plausible?}; it should ask,
\emph{Can a policy exploit this transition to create an unrealistic,
  high-value trajectory?}

To bridge this gap from local plausibility to global strategic
robustness, we argue for a new objective. We reframe simulator
learning not as a prediction task, but as a two-player, zero-sum game
between a \emph{Model Player} (the simulator) and a \emph{Policy
  Player} (the adversary). The Policy Player's goal is to find the
policy $\pi$ that reveals the largest discrepancy between the true
value in the real world $V_\pi(M)$ and the predicted value in the
simulator $V_\pi(\h M)$. The Model Player's goal is to update its
parameters to minimize this worst-case discrepancy.

This leads to our central minimax objective:
\begin{align}\label{eqn:main_objective}
  \min_{\h M \in \cM} \max_{\pi \in \Pi} \, \abs*{ V_{\pi}(M) - V_{\pi}(\h M) }
\end{align}
By solving this game, we produce a simulator that is not just accurate on
average, but is specifically robust against being exploited by the very agent
that will be trained within it.  This formulation also provides a principled
theoretical foundation for related \emph{proxy hacking} problems, such as reward
hacking in RLHF \citep{GuptaFischDannAgarwal2025, christiano2017deep} and
out-of-distribution exploitation in offline RL \citep{levine2020offline}.  Our
contributions are as follows:
\begin{itemize}

\item \textbf{A Policy-Aware Minimax Objective:} We formulate robust simulator
  learning as a minimax game and connect it to a practical online learning
  framework. We provide a regret bound (Theorem~\ref{th:main}) showing that if
  the model player achieves sublinear regret, its average value-gap against the
  adversary also vanishes.
    
\item \textbf{Tractable Critic-Based Guarantees:} We show that the intractable
  policy-maximization $\max_\pi$ in \eqref{eqn:main_objective} can be relaxed to
  a tractable $\max_D$ over a critic function $D$. We provide finite-time
  guarantees (Lemma~\ref{lemma:general-coverage-bound}) that bound the true
  policy-value gap by the loss of this critic-based game, scaled by a
  \emph{coverage constant} $\kappa$.
    
\item \textbf{The Error-MDP Duality:} We introduce our most significant
  theoretical novelty. We prove (Theorem~\ref{th:error-mdp-duality}) that the
  \emph{global} problem of finding the worst-case policy $\pi$ is formally dual
  to a \emph{local} RL problem: finding the optimal policy in an
  \emph{Error-MDP} where the reward is precisely the one-step critic error.
    
\item \textbf{Provably Stable Active Learning:} We resolve the circular
  dependency between learning a model and collecting its counter-examples.  We
  formalize active data selection as an iterative game and prove its
  \emph{convergence to a saddle point} (Theorem~\ref{th:convergence}).  This
  guarantees that our critic-guided sampling mechanism is not merely heuristic,
  but a theoretically stable method for minimizing the coverage
  constant~$\kappa$.
  
\item \textbf{A General Framework for Proxy Hacking:} We show this framework is
  a general solution for \emph{proxy hacking}, with direct applications to
  solving reward hacking in RLHF and OOD-exploitation in Offline RL
  (Section~\ref{sec:applications}).
\end{itemize}
\paragraph{Related Work.}
Our approach intersects with several key areas. Uncertainty-aware MBRL relies on
ensembles or pessimism to mitigate model error, whereas we actively seek and
correct strategic flaws. Adversarial RL typically perturbs dynamics to robustify
policies; in contrast, we perturb policies to robustify models. Finally,
Value-Aware Model Learning weights errors by value gradients for a fixed policy,
while our minimax formulation actively anticipates the worst-case policy
shift. We provide an extended discussion in
Appendix~\ref{sec:related_work_full}.

\paragraph{Paper Organization.}
The paper is organized as follows. Section~\ref{sec:sim-metrics} formalizes the
landscape of simulator fidelity metrics and justifies our
choice. Section~\ref{sec:online-learning-guarantees} presents our first main
result, an online-learning guarantee for the full minimax
objective. Section~\ref{sec:game-simplification} introduces our practical,
critic-based simplification of the game and provides robust value-gap
guarantees. Section~\ref{sec:active-data-selection} presents our novel
theoretical results on active data selection, including the Error-MDP duality
and convergence proofs. In Section~\ref{sec:experiments}, we present an
experimental validation of our theoretical results by examining a controlled 2D
scenario.  Section~\ref{sec:applications} discusses the broader applications of
our framework to other proxy-hacking problems.  Section~\ref{sec:joint_learning}
generalizes our framework to the joint learning of dynamics and rewards.

\section{Setup and Notation}

We work with \emph{discounted Markov decision processes} (MDP)
$(\sS, \sA, \sfP, r, \gamma)$ with state space $\sS$, action space $\sA$,
transition kernel $\sfP(\cdot \mid s, a)$, reward function
$r(s, a) \in [0, R_{\max}]$, and discount factor $\gamma \in (0, 1)$. For ease
of presentation, we assume the reward function $r$ is known and fixed, focusing
our attention on the challenge of learning the transition dynamics $\sfP$. We
treat the general case of jointly learning dynamics and rewards in
Section~\ref{sec:joint_learning}.

For a (possibly randomized) policy $\pi$ define the unnormalized
discounted state-action occupancy measure
$
  d_\pi(s, a)
  = \E\bracket*{\sum_{t = 0}^\infty \gamma^t \mathbf{1}\curl*{(s_t, a_t) = (s, a)}},
$
so that for any reward function $r$ we have
$V_\pi(\sfP) = \sum_{s, a} d_\pi(s, a) r(s, a)$, the expectation being with
respect to trajectories produced by $\pi$ and transitions $\sfP$.  For a
candidate model (simulator) $\h \sfP$ we write $V_\pi(\h \sfP)$ for the value of
$\pi$ computed when transitions are sampled from $\h \sfP$.

We assume we are given a class of candidate simulators $\sM$, each consisting of
a transition kernel $\h \sfP(\cdot\mid s, a)$ as well as a set of policies
$\Pi \subseteq \{ \sS \rightarrow \Delta(\sA)\}$ of possible interest, which can
be the set of all policies. We seek a simulator $\h \sfP$ which allows us to
accurately determine the value of all policies of interest $\pi \in \Pi$, which
directly leads to the main objective in
Equation~\eqref{eqn:main_objective}. While there are other possible desiderata
for simulator learning, like probability matching of full trajectories, these
often lead to infeasible learning or limited usability of the simulator. We
provide a detailed discussion of desiderata the next section.

\section{What Should a \emph{Good} Simulator Match?
  Metrics and Desiderata}
\label{sec:sim-metrics}

Before formalizing our minimax objective, we must first establish what
it means for a simulator to be \emph{good}. In standard supervised
learning, fidelity is measured by predictive accuracy on a fixed
dataset. In reinforcement learning, however, fidelity is measured by
the success of policies trained within the simulator when deployed in
the real world. This distinction requires a careful choice of metric.

For clarity, throughout the main analysis we identify the true world
model $M$ with its transition kernel $\sfP$, and the simulator $\h M$
with $\h \sfP$, treating the reward function $r$ as known. 

A (possibly randomized, stationary) policy $\pi$ induces a trajectory
distribution over infinite sequences
$\tau = (s_0, a_0, s_1, a_1, \ldots)$ under transition kernel $\sfP$
and initial state law $\rho_0$; we write $\Pr_\sfP^\pi(\tau)$ for this
distribution. We denote by $P_\sfP^\pi$ the marginal distribution over
finite trajectory prefixes (e.g.\ length-$H$ segments). The discounted
state-action occupancy for $\pi$ under $\sfP$ is
\[
d_\pi^\sfP(s, a)
= \E\bracket*{\sum_{t=0}^\infty \gamma^t
  \mathbf 1\{(s_t, a_t) = (s, a)\}},
\]
and the one-step conditional transition at $(s, a)$ is
$\sfP(\cdot \mid s, a)$.

Below we consider several notions of simulator fidelity, ranging from
the strongest (trajectory-level) to the weakest (optimal-value only).

\subsection{Trajectory or \emph{reward-free} closeness}
\label{sec:A}

\begin{definition}
We say $\h \sfP$ is \emph{trajectory-close} to $\sfP$ (for horizon
$H$) if some divergence between their finite-horizon trajectory
distributions is small:
\[
D_{\mathrm{traj}}^{(H)}(\sfP, \h \sfP)
= \sup_{\pi \in \Pi} D\paren*{P_\sfP^\pi(\tau_{0:H-1}), P_{\h \sfP}^\pi(\tau_{0:H-1})},
\]
where $D$ is a
divergence such as TV, KL, or Wasserstein.
\end{definition}

This is a \emph{reward-free} criterion: if trajectory distributions
match, then all reward functions measurable on trajectories, including
state- or action-dependent ones, induce the same return distributions.
It is extremely strong: matching trajectory distributions implies
matching all marginals, including $(s, a)$ occupancies and one-step
conditionals, but is infeasible for large or continuous state spaces.
For discounted infinite-horizon problems one can instead consider
the distribution over discounted trajectories.
In practice, one truncates to large $H$ and weights by
$\gamma^t$. This metric corresponds to the ideal of a ``Digital
Twin'', that is a simulator that captures the physics of the
environment so indistinguishably that an outside observer cannot tell
the difference. While desirable, we argue that when the model class is
misspecified (limited capacity), striving for this physical perfection
is inefficient compared to optimizing for strategic equivalence.

\subsection{Occupancy or one-step triplet matching}
\label{sec:B}

\paragraph{Definition.}
We say $\h \sfP$ matches \emph{one-step dynamics} with respect to an
occupancy family $\cD$ if, for $(s, a)$ in the support of $\cD$ (e.g.\
those with nonnegligible $d_\pi^\sfP(s, a)$ for relevant $\pi$),
\[
  D_{\text{one-step}}\paren*{\sfP(\cdot \mid s, a), \h \sfP(\cdot \mid s, a)}
  = D(\sfP(\cdot \mid s, a), \h \sfP(\cdot \mid s, a))
\]
is small, where $D_{\text{one-step}}$ is a divergence on next-state
distributions (TV, KL, Wasserstein, etc.). We can typically aggregate
this as
\[
\E_{(s, a) \sim \mu}
\bracket*{D_{\text{one-step}}(\sfP(\cdot \mid s, a),
  \h \sfP (\cdot \mid s, a))},
\]
for an occupancy $\mu$, for instance $d_\pi^\sfP$ or an average over
several policies.

Matching one-step conditionals for all $(s, a)$ implies trajectory
closeness (via simulation lemmas). However, the converse does not
strictly hold in the context of learning: matching the marginal
trajectory distribution of a specific policy does not necessarily
recover the true conditional transition kernel at states that are
rarely or never visited by that policy.
In continuous spaces, using weak metrics such as Wasserstein-$1$ can
be far more stable and sample-efficient than TV or KL, and respects
the geometry of $\sS$.
A natural variant is to match occupancies over triplets $(s, a, s')$,
capturing both the action selection and next-state distribution
induced by $\pi$.

\subsection{Policy-value matching}
\label{sec:C}

\paragraph{Definition.}
Require that for every policy $\pi$ in some class $\Pi$, the following
inequality holds:
\[
\abs*{V_\pi^\sfP - V_\pi^{\h \sfP}} \leq \e.
\]
Variants include restricting $\Pi$ to deterministic stationary,
randomized stationary, or near-optimal policies.
The following is a straightforward result about deterministic versus
randomized policies.

\begin{proposition}
If $V_\pi^\sfP = V_\pi^{\h \sfP}$ for every deterministic stationary
policy $\pi$, then equality also holds for all stationary randomized
policies.
\end{proposition}

\begin{proof}
A stationary randomized policy is a state-wise mixture of deterministic
policies. The value $V_\pi$ is linear in the induced occupancy
$d_\pi$, which itself depends linearly on the action mixture. Equality
on deterministic policies (the extreme points) thus extends by
convexity to their mixtures.
\end{proof}
Note that this argument applies to stationary state-wise
randomization. For history-dependent policies, the equivalence may
fail, but the restriction suffices for most theoretical analyses.
This restriction simplifies the theoretical landscape
significantly. For example, it allows us to replace complex stochastic
reward distributions with their constant expected values without loss
of generality, as the value function depends only on the expectation.

\subsection{Optimal-policy-value matching}
\label{sec:D}

\paragraph{Definition.}
Require that optimal values coincide (or are close):
\[
\big|V_\sfP^* - V_{\h \sfP}^*\big| \leq \e,
\qquad V^* = \sup_{\pi} V_\pi.
\]
This is the weakest requirement: useful for value estimation but
insufficient for policy transfer.  Two MDPs may share the same optimal
value yet have disjoint sets of near-optimal policies; hence matching
$V^*$ alone gives little guarantee of behavioral fidelity. For
example, in a Multi-Armed Bandit, this criterion only requires the
optimal arm to have the correct expected value; sub-optimal arms could
be modeled with arbitrarily incorrect values. If the environment
dynamics shift slightly, a previously sub-optimal action might become
optimal, but the model would provide no useful guidance.

\subsection{Implications and recommendations}

The logical relationships between the criteria are straightforward:
trajectory closeness (Subsection~\ref{sec:A}) implies one-step
matching (Subsection~\ref{sec:B}), which in turn implies policy-value
matching (Subsection~\ref{sec:C}), as follows from standard simulation
lemmas. The converse directions generally fail unless the dynamics
satisfy additional structural constraints. Furthermore, equality of
values for all deterministic stationary policies already guarantees
equality for all stationary randomized policies, owing to linearity in
state-action occupancies. By contrast, requiring only that optimal
values coincide, as in (Subsection~\ref{sec:D}), is far too weak to
ensure the successful transfer of policies across environments.

From both a theoretical and a practical perspective, the most
promising compromise seems to be to train simulators using
\emph{occupancy-weighted, value-aware one-step losses}. Specifically,
we can seek to minimize
\[
\min_{\h \sfP}
  \E_{(s, a)\sim \mu}
  \bracket*{D_{\text{one-step}}(\sfP(\cdot\mid s, a), \h \sfP(\cdot \mid s, a)) },
\]
where $\mu$ is an occupancy measure concentrating on the regions of
the state--action space visited by policies of interest.  In the
theoretical analysis of Section~\ref{sec:online-learning-guarantees},
choosing $D_{\text{one-step}} = \KL$ yields the occupancy-weighted KL
divergence $\E_{d_\pi}[\KL(\sfP\|\h \sfP)]$, which not only provides
an upper bound on the policy value gap but also admits an
interpretation as a regret-like objective in an online-learning
framework.

In continuous or high-dimensional state spaces, it is often preferable
to replace the KL divergence with weaker, geometry-aware distances such
as the Wasserstein or Sinkhorn metrics. These respect the underlying
geometry of the state space, are more robust to estimation noise, and
can be combined with KL-based analysis to obtain hybrid surrogates that
balance tractability and fidelity.

In practice, model learning should focus on the \emph{reachable
  region} of the state space, that is, the subset of states visited by
near-optimal or adversarial policies. Matching trajectories globally
is typically infeasible and unnecessary; fidelity within this region
is what ultimately determines policy robustness and transfer
performance.

Finally, of course, empirical evaluations could report several
complementary metrics, including the occupancy-weighted one-step KL or
TV divergence, short-horizon trajectory Wasserstein distance, and the
worst-case value-gap over a held-out policy set. Each captures
different aspects of simulator fidelity and exposes distinct failure
modes.

Overall, the most useful and theoretically grounded criterion for
decision-making seems to be an occupancy-weighted, value-aware
one-step divergence, typically the expected KL divergence under
relevant occupancies, or a Wasserstein analogue for continuous spaces.
This criterion is weaker, and therefore more practical, than full
trajectory matching, yet substantially stronger and more
decision-relevant than aligning only the optimal values.
With these metrics in hand, we now formalize the minimax game that
optimizes for strategic robustness.

\textbf{Example: case of minimal MDPs.}
Consider a true environment $\sfM$ with high-entropy, irrelevant noise
(e.g., a TV playing static in the background) and its minimal
equivalent $\sfM'$ (a bisimulation that ignores the static).  Under a
\emph{reward-free} metric like trajectory KL (Subsection~\ref{sec:A}),
$\sfM'$ is a poor simulator because it fails to model the complex
noise of $\sfM$.  But, under our \emph{strategic} objective
(Subsection~\ref{sec:C}), $\sfM'$ is a \emph{perfect} simulator, as
$V_\pi(\sfM) = V_\pi(\sfM')$ for all $\pi$. This illustrates why
predictive accuracy (modeling the noise) is often misaligned with
robust control. Our framework favors $\sfM'$, dedicating capacity only
to dynamics that affect the value function.

A concrete example is a Block MDP, where the control problem can be
solved entirely in the latent space. A reward-free metric might waste
capacity modeling the complex, high-dimensional mapping from latents
to observations, whereas our strategic objective focuses solely on the
latent dynamics required to estimate value.

\section{Game Formulation using Likelihood Losses}
\label{sec:online-learning-guarantees}
For a fixed policy $\pi$ and a kernel
$\h \sfP$ define the (expected negative log-likelihood) loss
\[
  \cL_\pi(\h \sfP)
  = \E_{(s, a) \sim d_\pi}\bracket*{\E_{s' \sim \sfP(\cdot \mid s, a)}
    \bracket*{-\log \h \sfP(s'\mid s, a)}}.
\]
Equivalently
$\cL_\pi(\h \sfP) = \E_{d_\pi}[\KL(\sfP \parallel \h \sfP)] +
\E_{d_\pi}[H(\sfP(\cdot \mid s, a))]$, so minimizing $\cL_\pi$ over
$\h \sfP$ amounts to minimizing the occupancy-weighted KL.

We consider the following online interaction over $T$ rounds. In round
$t = 1, \ldots, T$, an adversary (policy-player) first chooses a policy $\pi_t$,
then the model-player outputs $\h \sfP_t \in \sM$ (possibly using past
information), and finally, the model-player then incurs loss
$\ell_t(\h \sfP_t) = \cL_{\pi_t}(\h \sfP_t)$.  Let $\Regret_T$ denote the
model-player's regret relative to the best fixed kernel in $\sM$:
\[
  \Regret_T
  = \sum_{t = 1}^T \ell_t(\h \sfP_t) - \min_{\h \sfP\in\sM}
  \sum_{t = 1}^T \ell_t(\h \sfP).
\]
We seek to relate the average value-gap seen on the adversarial
sequence of policies to the average occupancy-weighted KL losses and
the model-player's regret.

\begin{theorem}[Online-to-Minimax Guarantee for Simulator Learning]
\label{th:main}
Assume standard regularity conditions (see Appendix~\ref{sec:analysis_game_ll})
and suppose the model-player produces kernels
$\h \sfP_1, \ldots, \h \sfP_T\in\sM$ and attains regret $\Regret_T$ relative to
the losses $\ell_t(\h \sfP) = \cL_{\pi_t}(\h \sfP)$. Define the per-round value
gap as $ \Delta_t = \abs*{V_{\pi_t}(\sfP)-V_{\pi_t}(\h \sfP_t)}$.  Then the
average value-gap $\frac{1}{T}\sum_{t = 1}^T \Delta_t$ is bounded by
\[
\frac{\gamma R_{\max}}{(1 - \gamma)^{3/2}}\;
\sqrt{ \frac{1}{2} \paren*{ \min_{\h \sfP\in\sM}\frac{1}{T}\sum_{t = 1}^T \cL_{\pi_t}(\h \sfP) \;+\; \frac{\Regret_T}{T} }}.
\]
In particular, if the model-player runs a no-regret algorithm with
$\Regret_T = o(T)$ then the average value-gap converges (as
$T\to\infty$) to the intrinsic approximation term
$\frac{\gamma R_{\max}}{(1 - \gamma)^{3/2}} 
\sqrt{ \frac{1}{2} \min_{\h \sfP \in \sM} \frac{1}{T}\sum_{t = 1}^T
  \cL_{\pi_t}(\h \sfP) }.$

\end{theorem}

The theorem says that a model-player running any no-regret online
algorithm (e.g.\ Online Gradient Descent or Mirror Descent in the
conditional-distribution space) guarantees that the average value-gap
on the sequence of adversarial policies is controlled by two terms:
(i) the \emph{approximation term}, which is the best average
occupancy-weighted KL that any model in $\sM$ can achieve on the
adversarial occupancies $\{\pi_t\}$, and (ii) the model-player's
regret per round. If an adversary concentrates on the worst-case
policies and the model class $\sM$ is rich enough, the average gap
approaches the minimax value-gap.

The statement above assumes access to the exact occupancy-weighted
expected negative log-likelihoods $\cL_{\pi_t}(\h \sfP)$. In practice
these must be estimated from a finite number of samples, $m$ (either
by rolling out $\pi_t$ in the real world or by using an off-policy
estimator from a dataset). This finite-sample estimation of the loss
introduces an additional error term, often on the order of
$O(1/\sqrt{m})$, which must be incorporated into the analysis.

We note that standard MLE-based approaches suffer a slower
$O(T^{-1/4})$ convergence rate for control due to the geometry of KL
divergence (via Schützenberger-Pinsker), while our minimax formulation
targets the optimal $O(T^{-1/2})$ rate directly. We provide a rigorous
proof and discussion of this lower bound in
Appendix~\ref{sec:lower-bounds}.

\section{Game Formulation using Upper-Bounds}
\label{sec:game-simplification}

Since learning a simulator by maximum likelihood estimation is standard, it may
appear natural to approximate the objective in \eqref{eqn:main_objective} using
negative log likelihood losses. This indeed yields a valid approach which we
discussed and analyzed in the previous section. While this gives strong
guarantees (see Theorem~\ref{th:main}), it involves a maximization over the
entire policy space, which can be computationally expensive. Moreover,
evaluating the negative log likelihood loss requires knowledge of the true
dynamics $\sfP$ to compute the KL divergence.  We here instead present a
tractable simplification. We replace the intractable $\max_\pi$ player with a
feasible $\max_D$ player, a critic, and substitute the KL loss with a
dual-divergence loss that can be estimated directly from samples.

The difference in policy values under the true and learned dynamics is
upper-bounded by the worst-case single-step transition error, measured
by the Total Variation (TV) distance:
\begin{align*}
  \max_{\pi \in \Pi} \abs*{ V_{\pi}(\sfP) - V_{\pi}(\h \sfP) }\leq \frac{R_{\max}}{(1 - \gamma)^2}
  \sup_{x = (s, a)} \TV\bigl(\sfP(x), \h \sfP(x)\bigr),
\end{align*}
where $\sfP(x) = \sfP(s,a)$ is short-hand for $\sfP(\cdot | s, a)$. This
motivates a simplified, more conservative game: instead of defending against the
worst-case policy, the simulator defends against the worst-case local error in
predicting next-state transitions.
The TV distance admits the dual characterization
\[
  \TV(\sfP, \sfQ)
  = \frac{1}{2}\sup_{\|D\|_\infty \leq 1}
  \Big\{ \E_{x \sim \sfP}[D(x)] - \E_{x \sim \sfQ}[D(x)] \Big\},
\]
where absolute values are unnecessary, since the maximization allows
the critic $D$ to effectively learn the sign of the difference
$\sfP - \sfQ$.
This dual form naturally suggests an adversarial training setup:
replace the intractable supremum over $(s, a)$ with an expectation under
a data distribution $d_{\text{data}}(s, a)$, and represent the supremum
over bounded $D$ by an adversarial critic $D \in \cD$ whose output lies
in $[-1,1]$.  This leads to the following practical two-player game:
\[
  \min_{\h \sfP \in \sM} \max_{D \in \cD}
  \E_{(s, a) \sim d_{\text{data}}}
  \!\left[
    \Delta_{D, \h \sfP}(s, a)
  \right]\quad\textrm{ where }\quad
\Delta_{D, \h \sfP}(s, a) \!= \E_{s'\sim\sfP}[D(s, a, s')] \!-
\E_{s'\sim\h\sfP}[D(s, a, s')].
\]
To make the game tractable, we replace the policy maximization with a critic
maximization using the dual representations of Total Variation (TV) or
Wasserstein distances. This matches the definition of witnessed model misfit in
\citet{sun2019model}. The result is a GAN-like objective where the critic
distinguishes real transitions from simulated ones, and the model minimizes this
distinction.

The reformulation yields a concrete, trainable game that upper-bounds the
original objective from \eqref{eqn:main_objective}. However, the simplification
incurs a key trade-off: replacing the policy adversary with a critic focused on
local discrepancies removes the strategic, policy-aware aspect of the problem.
The resulting simulator is encouraged to be accurate \emph{everywhere} under
$d_{\text{data}}$, rather than focusing capacity on regions that policy
optimization would actually visit or exploit.  While substantially more
tractable, this objective can therefore be overly conservative and less
sample-efficient.  Nonetheless, under a mild \emph{coverage assumption} relating
$d_{\text{data}}$ to the occupancies of relevant policies, we can recover
guarantees on the original policy value-gap:

\begin{lemma}[General Value-Gap Bound]
\label{lemma:general-coverage-bound}
Let $\wt \sfD$ be any divergence between one-step transition
kernels. Let $C(\pi, \h \sfP, s, a)$ be a state-action dependent cost
term such that for any policy $\pi$ and model $\h \sfP$, the value-gap
is bounded by the $d_\pi$-weighted divergence:
\begin{align*}
&\abs*{V_\pi(\sfP) - V_\pi(\h\sfP)}
\leq \E_{(s, a)\sim d_\pi} \bracket*{ C(\pi, \h \sfP, s, a) \cdot \wt \sfD\paren*{\sfP(s, a),\h\sfP(s, a)} }.
\end{align*}
Furthermore, assume a $\kappa$-coverage data distribution
$d_{\mathrm{data}}$ exists such that
$d_\pi(s, a) \leq \kappa\,d_{\mathrm{data}}(s, a)$ for all
$\pi \in \Pi$.  Then the uniform value-gap is bounded by the
$d_{\mathrm{data}}$-weighted divergence:
\begin{align}
\label{eqn:bound_general}
&\sup_{\pi\in\Pi}\abs*{V_\pi(\sfP) - V_\pi(\h\sfP)}
\leq \kappa \cdot \sup_{\pi\in\Pi} \E_{(s, a)\sim d_{\mathrm{data}}} \bracket*{ C(\pi, \h \sfP, s, a) \ \wt \sfD\paren*{\sfP(s, a),\h\sfP(s, a)} }.
\end{align}
\end{lemma}

This general lemma is the core mechanism of our practical
framework. The challenge is reduced to finding an appropriate metric
$\wt \sfD$ and cost $C$ that satisfy the first premise.
Essentially, this lemma acts as a transfer mechanism: the constant
$\kappa$ functions as a Radon-Nikodym derivative, allowing us to
transfer the error bound from the unknown adversarial occupancy
$d_\pi$ to the controlled data distribution $d_{\mathrm{data}}$. We
now present two such instantiations.

\subsection{TV-critic -- discrete spaces}
\label{sec:tv_critic}

For discrete state spaces, we can use the Total Variation (TV) distance as
$\wt \sfD$. This motivates the \emph{TV-Critic Game}. By the dual
characterization of TV, we have
$\TV(\sfP, \sfQ) = \frac{1}{2}\sup_{\|D\|_\infty \leq 1} \E_\sfP[D(x)] -
\E_\sfQ[D(x)]$.  Plugging this into Equation~\eqref{eqn:bound_general} our
objective becomes minimizing the loss of a data-weighted GAN:
\[
  \min_{\h \sfP \in \sM} \max_{D \in \cD}
  \E_{(s, a) \sim d_{\text{data}}} \bracket*{ \Delta_{D, \h \sfP}(s, a) }.
\]
Here $\cD$ is a critic class approximating the $\|D\|_\infty \leq 1$
ball. The following theorem gives a full guarantee for the solution of
this game, separating the error sources.

\begin{theorem}[Minimax / empirical minimizer guarantees - TV case]
\label{th:minimax_tv}
Assume rewards $r(s, a)\in[0,R_{\max}]$ and discount $\gamma\in(0,1)$.
Let $\Pi$ be a policy class of interest, and assume
$d_{\mathrm{data}}$ satisfies $\kappa$-coverage for $\Pi$. Let
$J_{\cD}(\h\sfP) = \E_{d_{\mathrm{data}}}[\sup_{D \in \cD} \Delta_{D, \h \sfP}(s,
a)]$ be the value of the TV-Critic Game, and let
$\e_{\mathrm{critic}}\geq 0$ be the error from approximating the
$L_\infty$ ball with $\cD$.
\begin{itemize}
\item (A) Population minimizer.  If $\h\sfP^*$ achieves
  $J_{\cD}(\h\sfP^*) \leq \e_{\mathrm{model}}$, then
  $\sup_{\pi\in\Pi}\abs{V_\pi(\sfP) - V_\pi(\h\sfP^*)}$ is upper-bounded by
\[
\frac{\gamma R_{\max}}{1 - \gamma}\; \kappa \;
\Big( \tfrac{1}{2} \e_{\mathrm{model}} + \tfrac{1}{2}\e_{\mathrm{critic}} \Big).
\]

\item (B) Finite-sample minimizer.  If $\h\sfP$ is an $\eta$-approximate
  minimizer of an empirical objective $\h J_N$ (from $N$ samples) which has
  uniform convergence error $\omega$, then with probability at least $1-\delta$,
  $\sup_{\pi\in\Pi}\abs*{V_\pi(\sfP) - V_\pi(\h\sfP)}$ is upper-bounded by
\[
\frac{\gamma R_{\max}}{1 - \gamma}\; \kappa \;
\bracket*{ \tfrac{1}{2}
\paren*{\inf_{\h\sfP\in\sM}J_{\cD}(\h\sfP) + 2\omega + \eta} + \tfrac{1}{2}\e_{\mathrm{critic}} }.
\]
\end{itemize}
\end{theorem}

\subsection{TV Critic Game Implementation}

We propose a practical implementation of this game that alternates
between updating the critic and the model, similar to Generative
Adversarial Networks (GANs). In each round, we sample a minibatch of
transitions from the data distribution $d_{\text{data}}$. The critic
$D_\phi$ is updated via stochastic gradient ascent to maximize the
empirical difference between real and predicted next-states,
approximating the supremum in the dual TV formulation. Simultaneously,
the model $\h \sfP_\theta$ is updated via stochastic gradient descent
to minimize this same quantity, effectively trying to "fool" the
critic into thinking its generated transitions are real.

\begin{algorithm}[t]
\caption{TV-Critic Minimax Training}
\label{alg:tv-critic}
\begin{algorithmic}[1]
\REQUIRE Data $d_{\text{data}}$, true kernel $\sfP$, model $\sM_\theta$, critic $\cD_\phi$, learning rates $\eta_\theta, \eta_\phi$, batch size $B$.
\STATE Initialize model $\theta$ and critic $\phi$.
\FOR{$t = 1, \dots, T$}
  \STATE Sample minibatch $\{(s_i, a_i)\}_{i=1}^B \sim d_{\text{data}}$.
  \STATE For each $i$, sample $s'_i \sim \sfP(\cdot \mid s_i, a_i)$ and $\h s'_i \sim \h\sfP_\theta(\cdot \mid s_i, a_i)$.
  \STATE \textbf{Critic Update:}
  \[
  \phi \gets \phi + \eta_\phi \frac{1}{B} \sum_{i=1}^B \nabla_\phi \bracket*{ D_\phi(s_i, a_i, s'_i) - D_\phi(s_i, a_i, \h s'_i) }
  \]
  \STATE \textbf{Model Update:}
  \[
  \theta \gets \theta - \eta_\theta \frac{1}{B} \sum_{i=1}^B \nabla_\theta \bracket*{ - D_\phi(s_i, a_i, \h s'_i) }
  \]
\ENDFOR
\STATE \textbf{Return:} Averaged model parameters $\ov\theta = \frac{1}{T} \sum \theta_t$.
\end{algorithmic}
\end{algorithm}

\subsection{Wasserstein-critic -- continuous spaces}
\label{sec:w1_critic}

For continuous state spaces, TV is too strong and $R_{\max}/(1 - \gamma)$ is a
loose, uniform bound. Furthermore, standard divergences like KL or TV often
suffer from vanishing gradients when the support of the learned model $\h \sfP$
and true dynamics $\sfP$ are disjoint \citep{arjovsky2017wasserstein}. A
geometry-aware metric like the 1-Wasserstein distance ($W_1$) is preferable as
it provides meaningful gradients even for disjoint distributions, ensuring the
model player receives a learning signal everywhere.  By Kantorovich-Rubinstein
duality, the $W_1$ distance is the supremum over 1-Lipschitz critics. The
resulting \emph{Wasserstein-Critic Game} objective is:
\[
  \min_{\h \sfP \in \sM} \max_{D \in \cD_{\text{Lip}}}
  \E_{(s, a) \sim d_{\text{data}}}
  \!\left[\Delta_{D, \h \sfP}(s,a)
  \right],
\]
where $\cD_{\text{Lip}}$ is a critic class approximating the 1-Lipschitz ball.

\begin{theorem}[Minimax guarantees - Wasserstein case]
\label{th:minimax_wasserstein}
Let $(\sS, d)$ be a metric state space.  Suppose for every policy $\pi \in \Pi$,
the expected value of the next state, denoted
$f(s, a) = \E_{s' \sim \h \sfP(\cdot \mid s, a)}[V_\pi^{\h \sfP}(s')]$, is
$L_v$-Lipschitz with respect to state $s$. Let
$J_{\cD_{\text{Lip}}}(\h\sfP) = \E_{d_{\mathrm{data}}}[\sup_{D \in
  \cD_{\text{Lip}}} \Delta_{D, \h \sfP}(s, a)]$ be the value of the
Wasserstein-Critic Game, and let $\e_{\mathrm{critic}} \geq 0$ be the error from
approximating the 1-Lipschitz ball with $\cD_{\text{Lip}}$.  If $\h\sfP^*$
achieves $J_{\cD_{\text{Lip}}}(\h\sfP^*) \leq \e_{\mathrm{model}}$, then:
\[
\sup_{\pi\in\Pi}\abs{V_\pi(\sfP) - V_\pi(\h\sfP^*)}
\le
\gamma L_v \kappa \paren*{ \e_{\mathrm{model}} + \e_{\mathrm{critic}} }.
\]
\end{theorem}

\textbf{Practical remarks and caveats.}
The guarantees for the Wasserstein game mirror those in
Theorem~\ref{th:minimax_tv}, but with the pre-factor
$\gamma L_v \kappa$ replacing
$\frac{\gamma R_{\max}}{1 - \gamma}\kappa$. The Value Lipschitz constant
$L_v$ can be bounded in terms of reward and transition Lipschitzness
(see Appendix~\ref{app:lipschitz} for a proof that Lipschitzness can
be derived from standard smoothness assumptions about the reward
function and transition dynamics). In all cases, the \emph{coverage
  constant $\kappa$} remains the critical bottleneck: if the data
distribution $d_{\mathrm{data}}$ does not cover the states an
adversary will visit, the guarantee is void. This motivates a method
for actively selecting $d_{\mathrm{data}}$.

While simple Empirical Risk Minimization (ERM) minimizes average
error, it may waste capacity on noise. In contrast, the Minimax
formulation serves as a hard-negative mining mechanism, focusing model
capacity specifically on dynamics that are currently being exploited.

\paragraph{Algorithmic solutions to TV critic and Wasserstein critic game.} The
objective of both games can be solved efficiently by alternate updates of the
critic and the model, similar to Generative Adversarial Networks. We provide
details on this in the appendix.

\subsection{Wasserstein Guarantee}

\begin{corollary}[Wasserstein Coverage Guarantee]
\label{cor:coverage-wasserstein}
Assume the setup of Lemma~\ref{lemma:general-coverage-bound}. Let $(\sS, d)$ be
a metric state space.  Suppose for every policy $\pi \in \Pi$, the expected
value of the next state, denoted
$f(s, a) = \E_{s' \sim \h \sfP(\cdot \mid s, a)}[V_\pi^{\h \sfP}(s')]$, is
$L_v$-Lipschitz with respect to state $s$.
Then, the following inequality holds:
\[
\sup_{\pi\in\Pi} \abs*{V_\pi(\sfP) - V_\pi(\h\sfP)}
\leq \gamma \, L_v \, \kappa \, \E_{(s, a)\sim d_{\mathrm{data}}}
\bracket*{W_1(\sfP(\cdot \mid s, a), \h\sfP(\cdot\mid s, a))}.
\]
\end{corollary}
See the appendix for the proof that Lipschitz dynamics plus rewards
imply Lipschitz Values.

\begin{proof}
  This is an instantiation of
  Lemma~\ref{lemma:general-coverage-bound}.  We start from the
  single-step error decomposition from the proof of
  Lemma~\ref{lemma:sim}:
  $\abs{V_\pi(\sfP) - V_\pi(\h \sfP)} \leq \gamma \E_{(s, a) \sim
    d_\pi} \bracket*{ \abs{\e(s, a)} }$, where
  $\e(s, a) = \E_{s' \sim \sfP}[V_\pi^{\h \sfP}(s')] - \E_{s' \sim \h
    \sfP}[V_\pi^{\h \sfP}(s')]$.  By Kantorovich-Rubinstein duality,
  since $V_\pi^{\h \sfP}$ is $L_v$-Lipschitz, this error is bounded:
  $\abs{\e(s, a)} \leq L_v \cdot W_1(\sfP(\cdot|s, a), \h
  \sfP(\cdot|s, a))$.  This gives us
  $\wt \sfD = W_1(\sfP, \h \sfP)$ and a state-independent
  (but policy-dependent) cost $C(\pi, \h \sfP, s, a) = \gamma L_v$.
  Plugging into Lemma~\ref{lemma:general-coverage-bound}:
\begin{align*}
\sup_{\pi\in\Pi}\abs*{V_\pi(\sfP) - V_\pi(\h\sfP)}
& \leq \kappa \cdot \sup_{\pi\in\Pi} \E_{(s, a)\sim d_{\mathrm{data}}} \bracket*{ \gamma L_v \cdot W_1(\sfP, \h\sfP) } \\
& = \gamma \kappa \sup_{\pi\in\Pi} \paren*{ L_v \E_{(s, a)\sim d_{\mathrm{data}}} \bracket*{ W_1(\sfP, \h\sfP) } }.
\end{align*}
Assuming $L_v$ is a uniform bound across all $\pi \in \Pi$ (a standard
assumption), we can move the $\sup$ inside, and since the expectation
no longer depends on $\pi$, the $\sup$ vanishes, giving the result.
\end{proof}

\subsection{Wasserstein critic implementation}

Implementing the Wasserstein game requires constraining the critic to
be 1-Lipschitz. We adopt standard techniques from Wasserstein GANs,
such as Gradient Penalty (GP) or spectral normalization, to enforce
this constraint softly. A key practical difference from the TV case is
the need to train the critic to near-optimality to ensure accurate
estimation of the $W_1$ distance before updating the
model. Algorithm~\ref{alg:wasserstein-critic} reflects this by
including an inner loop (repeated $n_D$ times) for the critic update.

\begin{algorithm}[H]
\caption{Wasserstein-Critic Minimax Training}
\label{alg:wasserstein-critic}
\begin{algorithmic}[1]
\REQUIRE Data distribution $d_{\text{data}}$, true kernel $\sfP$,
  model class $\sM_\Phi$, critic class $\cD_\Theta$, GP coefficient
  $\lambda$, learning rates $\eta_\phi,\eta_\theta$, critic steps
  $n_D$, total iterations $T$.

\STATE Initialize $\phi_1$ (model params), $\theta_1$ (critic params).

\FOR{$t=1,\dots,T$}

\STATE Sample minibatch $\{(s_i,a_i)\}_{i = 1}^B\sim d_{\text{data}}$.

\STATE For each $(s_i,a_i)$ sample real
$s'_i\sim\sfP(\cdot\mid s_i,a_i)$ and generated
$\h s'_i\sim\h\sfP_{\phi_t}(\cdot\mid s_i,a_i)$.

\FOR{$k = 1,\dots,n_D$}

\STATE Sample interpolation
$\tilde s_i = \alpha s'_i + (1-\alpha)\h s'_i$ with
$\alpha\sim\mathrm{Unif}(0,1)$.

\STATE Critic gradient step on
\[
  L_D(\theta)
  = -\frac{1}{B}\sum_{i = 1}^B
  \big( D_\theta(s_i,a_i, s'_i) - D_\theta(s_i,a_i,\h s'_i) \big)
  + \lambda \frac{1}{B}\sum_{i = 1}^B
  \bracket[\Big]{ \norm*{\nabla_{\tilde s_i} D_\theta(s_i,a_i,\tilde s_i)}_2 - 1}^2.
\]

\STATE $\theta \leftarrow \theta - \eta_\theta \nabla_\theta L_D(\theta)$.

\ENDFOR

\STATE Model update: minimize critic score on generated samples
\[
  L_M(\phi) = \frac{1}{B}\sum_{i = 1}^B D_{\theta}(s_i,a_i,\h s'_i)
  \quad\text{(optionally }
  + \lambda_{\text{lik}} \, \wt \cL_{\text{MLE}}(\phi)\text{)}.
\]

\STATE $\phi \leftarrow \phi - \eta_\phi \nabla_\phi L_M(\phi)$.

\ENDFOR

\STATE Return averaged model parameters
$\ov\phi = \tfrac{1}{T}\sum_{t=1}^T \phi_t$ and critic
$\ov\theta$.
\end{algorithmic}
\end{algorithm}

\subsection{Robustness to model misspecification}
\label{sec:misspecification}

A critical advantage of the minimax formulation arises when the model
class $\sM$ is \emph{misspecified}, meaning it does not contain the
true dynamics $\sfP$ (i.e.,
$\inf_{\h \sfP \in \sM} \KL(\sfP \parallel \h \sfP) > 0$). Standard
maximum likelihood estimation seeks the projection of $\sfP$ onto
$\sM$ that minimizes the global KL divergence. However, this
projection may sacrifice accuracy in strategic regions to fit
irrelevant stochasticity elsewhere, as discussed in the Minimal MDP
example.

In contrast, our framework serves as a \emph{selection rule} that
prioritizes value-relevant dynamics. We can formalize this as an
oracle inequality.

\begin{proposition}[Strategic Oracle Inequality]
  Let $\h \sfP_{\text{MM}}$ be the solution to the population minimax
  game (e.g., TV-critic) and let $\h \sfP_{\text{MLE}}$ be the
  asymptotic solution to maximum likelihood estimation. There exist
  environments $\sfP$ and model classes $\sM$ such that:
\[
  \sup_{\pi \in \Pi} \abs*{V_\pi(\sfP) - V_\pi(\h \sfP_{\text{MM}})} 
\ll \sup_{\pi \in \Pi} \abs*{V_\pi(\sfP) - V_\pi(\h \sfP_{\text{MLE}})}.
\]
\end{proposition}

\begin{proof}[Sketch]
  Consider the Minimal MDP case (Section~\ref{sec:sim-metrics}). Let
  $\sfP$ be a system with high-entropy irrelevant noise (e.g., a
  background TV screen) and deterministic relevant dynamics. Let $\sM$
  be a class of deterministic models with limited capacity.
\begin{itemize}

\item $\h \sfP_{\text{MLE}}$ attempts to match the high entropy of
  $\sfP$. Since $\sM$ is deterministic, $\h \sfP_{\text{MLE}}$ might
  oscillate or predict the mean of the noise, potentially distorting
  the relevant dynamics to minimize the global KL/MSE, leading to high
  strategic error.

\item $\h \sfP_{\text{MM}}$ minimizes the worst-case value gap. Since
  the noise does not affect $V_\pi$, the minimax solution will ignore
  the noise and perfectly fit the relevant deterministic dynamics,
  achieving zero value gap despite high KL divergence.
\end{itemize}
Thus, under misspecification, the minimax objective selects the model
that is most useful rather than most likely.
\end{proof}

\section{Active Data Selection for Improved Coverage Guarantees}
\label{sec:active-data-selection}

The coverage-based guarantees in Theorem~\ref{th:minimax_tv} and
\ref{th:minimax_wasserstein} are powerful, but they depend critically on the
coverage constant $\kappa$ of the data distribution $d_{\mathrm{data}}$ relative
to the policies of interest $\Pi$. This raises a chicken-and-egg problem: to get
a good guarantee, one must sample from a distribution $d_{\mathrm{data}}$ that
already covers the occupancies $d_\pi$ of the very policies $\pi$ that might
exploit the simulator.

This suggests an active learning approach, where the sampling
distribution $d_t$ is updated iteratively to find these exploitative
policies.
We now present our most significant theoretical insight.  We show that
the \emph{global} problem of finding the worst-case policy $\pi$ is
formally equivalent to a \emph{local} reinforcement learning problem:
finding the optimal policy in an \emph{Error-MDP} where the reward
signal is precisely the one-step modeling error.

We first define a new \emph{Error-MDP}, $\sfM_{\text{err}}$, which has
the same state space $\sS$, action space $\sA$, and \emph{true}
transition dynamics $\sfP$ as the real environment, but with a
different reward function.  The concept of using model error as an
intrinsic reward has been explored empirically in active exploration
methods like MCX \citep{swift2024mcx}; here, we provide the formal
duality proving that this strategy is not just a heuristic, but the
optimal response in a minimax robustness game.

Define the \emph{error-reward}
$r_{\text{err}}$ as the local, one-step modeling error, as measured by
our Wasserstein-based critic $r_{\text{err}}(s, a)
  = W_1\big(\sfP(s, a), \h \sfP(s, a)\big)$.
Let $V_{\pi}(\sfM_{\text{err}})$ be the expected (unnormalized)
discounted return for a policy $\pi$ in this $\sfM_{\text{err}}$
(the expectations are over $(s, a) \sim d_\pi$):
\begin{align*}
V_{\pi}(\sfM_{\text{err}})
= \E \bracket*{ r_{\text{err}}(s, a) }
= \E \bracket*{ W_1\big(\sfP(s, a),
  \h \sfP(s, a)\big) }.
\end{align*}
The following theorem proves that the worst-case policy value-gap (our
original objective) is directly upper-bounded by the optimal value of
this Error-MDP.

\begin{theorem}[Adversarial Policy as Error-Seeking Agent]
\label{th:error-mdp-duality}
Assume
$V_\pi^{\h \sfP}$ is $L_v$-Lipschitz for all $\pi \in \Pi$.  Let
$V^*(\sfM_{\text{err}}) = \sup_{\pi \in \Pi} V_{\pi}(\sfM_{\text{err}})$
be the optimal value of the Error-MDP $\sfM_{\text{err}}$.
Then, the worst-case policy value-gap is bounded by:
\[
\sup_{\pi \in \Pi} \abs*{V_\pi(\sfP) - V_\pi(\h \sfP)}
\leq \gamma \, L_v \, V^*(\sfM_{\text{err}}).
\]
\end{theorem}

The theorem implies that for any policy $\pi$, the performance gap
between the true MDP and the simulator is upper-bounded (up to a
constant) by the value of the corresponding Error-MDP.  Consequently,
a policy that achieves a large Error-MDP value is likely to also
induce a large discrepancy in returns.  This observation motivates
treating the Error-MDP as a proxy optimization problem for identifying
policies that expose weaknesses in the learned model, a standard RL
problem we can solve using familiar algorithms.

The Error-MDP provides a principled mechanism to identify where the
learned model $\h \sfP$ is inaccurate.  Its optimal policy naturally
seeks trajectories that accumulate large modeling errors, thereby
revealing parts of the state–action space that lack sufficient
coverage.  Using this policy for active sampling directs exploration
toward these high-error regions, improving model accuracy and
coverage, and thus tightening the worst-case performance bound
guaranteed by the theorem.

Note that our Error-MDP framework provides a theoretical justification for
\emph{intrinsic motivation} \citep{chentanez2005intrinsically}. While methods
like random network distillation (RND) \citep{burda2018exploration} use
prediction error as a heuristic for novelty, our
Theorem~\ref{th:error-mdp-duality} proves that maximizing model error is the
optimal adversarial strategy for robust simulator learning. This also suggests
that off-the-shelf exploration algorithms (like PPO with intrinsic rewards) can
be directly plugged in as the Policy Player in our game.

\subsection{Iterative critic-guided learning and its guarantees}

The Error-MDP duality justifies an iterative algorithm where we
alternate between (1) estimating the error-reward $r_{\text{err}}$
(the critic-game) and (2) sampling from policies $\pi$ that achieve
high returns in $\sfM_{\text{err}}$ to generate data $d_t$ for the
next round of model training.

Iterative Algorithm (Algorithm~\ref{alg:active-wasserstein}).
Leveraging the Error-MDP duality, we propose a practical iterative active
learning scheme. The algorithm maintains a current model $\h\sfP_t$ and updates
a Wasserstein critic $D_t$ to estimate the local error frontier
$r_{\text{err}}(s,a)$. It then uses an adversarial sampling policy to collect
new data $d_{t+1}$ specifically from regions where this error-reward is high. By
aggregating this hard data and retraining, the model is forced to resolve its
most critical "blind spots" in each round, progressively reducing the maximal
error available to the adversary.  While the critic objective is local, our
Error-MDP theory (Theorem~\ref{th:error-mdp-duality}) formally bounds the global
policy-value gap using this local critic’s loss. By optimizing the Error-MDP
reward, sampling distribution actively seeks state-action pairs that accumulate
these local errors, fundamentally connecting local transition errors to global
exploitability.

\begin{algorithm}[H]
\caption{Iterative Wasserstein Critic-Guided Model Learning}
\label{alg:active-wasserstein}
\begin{algorithmic}[1]
  \STATE \textbf{Input:} Initial model $\h \sfP_0$, $L_v$-Lipschitz
  critic class $\cD_{L_v}$, iterations $T$, initial data
  $d_{\mathrm{init}}$

\FOR{$t = 1$ to $T$}

\STATE Sample batch of transitions $(s, a, s')$ from current data
distribution $d_t$

\STATE Update critic $D_t \in \cD_{L_v}$ to maximize
$\E_{(s, a)\sim d_t} \big[ \Delta_{D, \h \sfP_{t-1}}(s,a) \big]$

\STATE Update model $\h \sfP_t$ to minimize
$\E_{(s, a)\sim d_t} \bracket*{ \E_{s'\sim \h \sfP_{t}}[-D_t(s, a, s')] }$
\STATE Define critic-guided distribution
$d_{t+1}(s, a) \propto \sup_{D(s, a,\cdot)\in\cD_{L_v}} 
  \Big| \Delta_{D, \h \sfP_t}(s,a) \Big|$
\ENDFOR

\STATE \textbf{Output:} Averaged model $\ov{\h \sfP}_T
= \frac{1}{T}\sum_{t=1}^T \h \sfP_t$
\end{algorithmic}
\end{algorithm}

This iterative process creates a feedback loop. We now provide two
guarantees for this algorithm. First, a finite-time bound on the
averaged model, and second, an asymptotic convergence guarantee.

\begin{theorem}[Finite-Time Guarantee for Iterative Learning]
\label{th:finite_time_active}
Assume the active data distributions $d_t$ from
Algorithm~\ref{alg:active-wasserstein} satisfy a uniform
$\kappa$-coverage for $\Pi$, and the value functions
$V_\pi^{\h \sfP_t}$ are uniformly $L_v$-Lipschitz. Let
$\e_{\mathrm{critic}}$ be the approximation error of $\cD_{L_v}$ to
the true $L_v$-Lipschitz ball. Then the time-averaged model
$\ov{\h \sfP}_T$ satisfies:
\[
\sup_{\pi\in\Pi} \abs*{V_\pi(\sfP) - V_\pi(\ov{\h \sfP}_T)}
\leq
\gamma L_v \kappa \paren*{\ov{J}_{\cD_{L_v}} + \e_{\mathrm{critic}}},
\]
where 
$\ov{J}_{\cD_{L_v}}
  \!\!\!= \frac{1}{T} \sum_{t=1}^T \E_{(s, a)\sim d_t}
  \bracket*{\sup_{D\in\cD_{L_v}} \curl*{
  \Delta_{D, \h \sfP_t}(s,a)}}$ is the time-averaged critic objective.
\end{theorem}

This theorem provides a finite-time bound but relies on a strong
uniform coverage assumption. We now show that this game, when
regularized, converges to a meaningful saddle point, providing a
justification for the algorithm's stability.

\begin{figure}[t]
  \centering
    \includegraphics[width=0.66\textwidth]{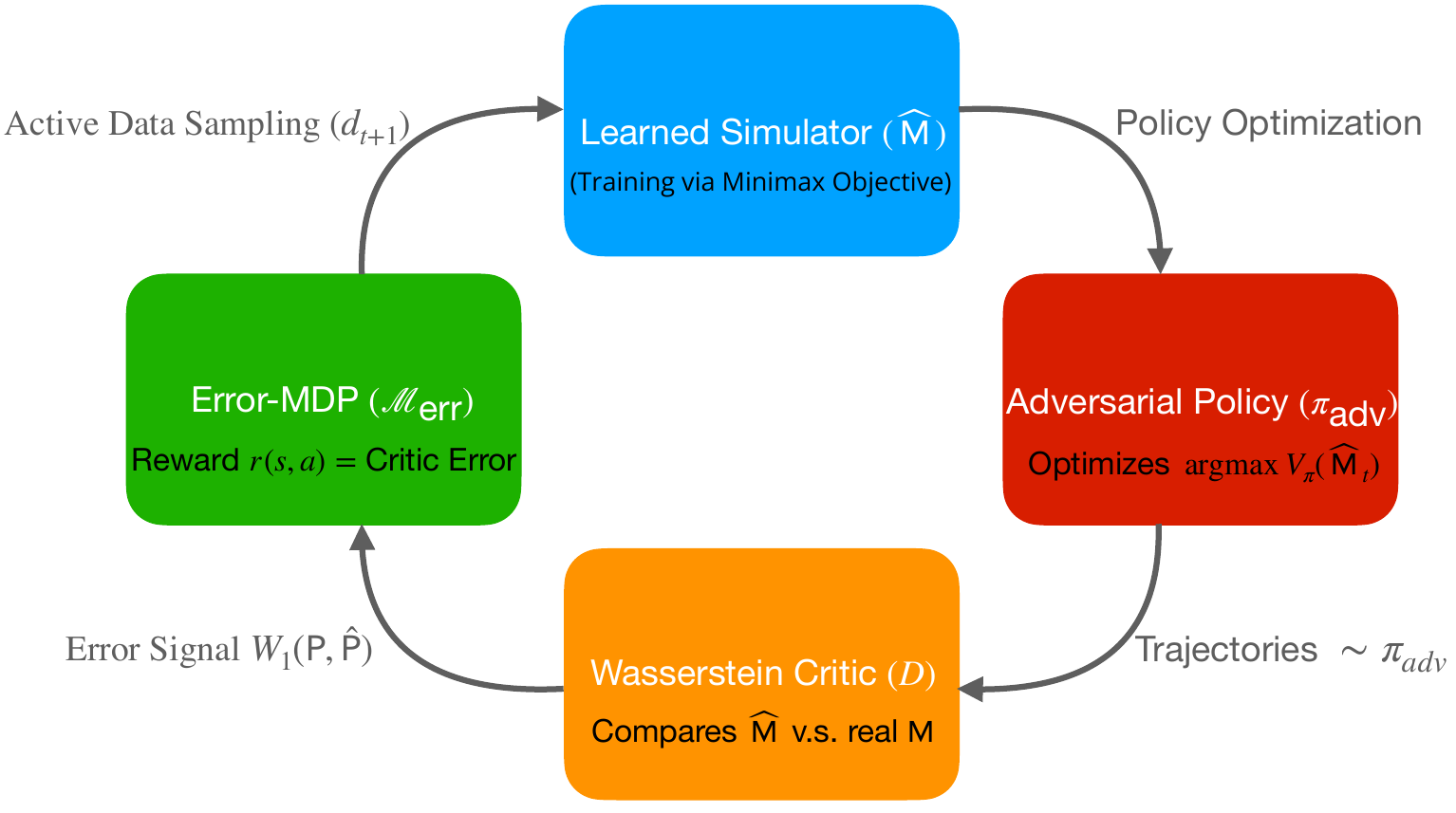}
    \caption{\textbf{The Error-MDP Feedback Loop.} Unlike standard
    active learning which uses heuristic uncertainty, our framework
    formally casts the model-error, estimated by the Critic, as the
    reward signal for a distinct RL problem, the Error-MDP, producing
    data that targets the specific strategic weaknesses of the current
    model. This visualizes the cycle described in
    Algorithm~\ref{alg:active-wasserstein}.}
    \label{fig:error_mdp_loop}
\end{figure}

\begin{theorem}[Convergence of Iterative Active Sampling Game]
\label{th:convergence}
Let the regularized payoff be
$L(\h\sfP, d) = \E_{d} [ W_1(\sfP, \h \sfP) ] - \lambda \, \KL(d
\parallel u)$, where $\lambda > 0$ and $u$ is a prior.  If the
Model-Player and Distribution-Player follow no-regret online
algorithms with regrets $\Regret_M(T)$ and $\Regret_D(T)$ respectively
against this payoff, then the time-averaged strategies
$(\ov{\h\sfP}_T, \ov{d}_T)$ converge to an approximate saddle point of
$L$:
\[
\max_{d \in \Delta} L(\ov{\h\sfP}_T, d) - \min_{\h\sfP \in \sM} L(\h\sfP, \ov{d}_T)
\leq \frac{\Regret_M(T) + \Regret_D(T)}{T}.
\]
As $T \to \infty$, if the average regrets vanish (i.e.,
$\Regret(T)/T \to 0$), the averaged strategies converge to the optimal
value of the regularized game.
\end{theorem}
This theorem proves that the iterative, co-adaptive algorithm is
stable and converges to a well-defined solution, providing a formal
basis for tuning the active sampling process (via $\lambda$) as a
trade-off between exploration (high $\lambda$) and exploitation (low
$\lambda$).

Although this section focuses on the Wasserstein-based formulation, an analogous
iterative scheme and convergence result can be derived for the Total Variation
(TV) case.

\begin{figure}[t]
  \centering
  \begin{tabular}{@{\hspace{0cm}}c@{\hspace{0.2cm}}c@{\hspace{0.3cm}}
    c@{\hspace{0cm}}c@{\hspace{0cm}}}
    \raisebox{-2.5mm}{\includegraphics[width=.315\textwidth]{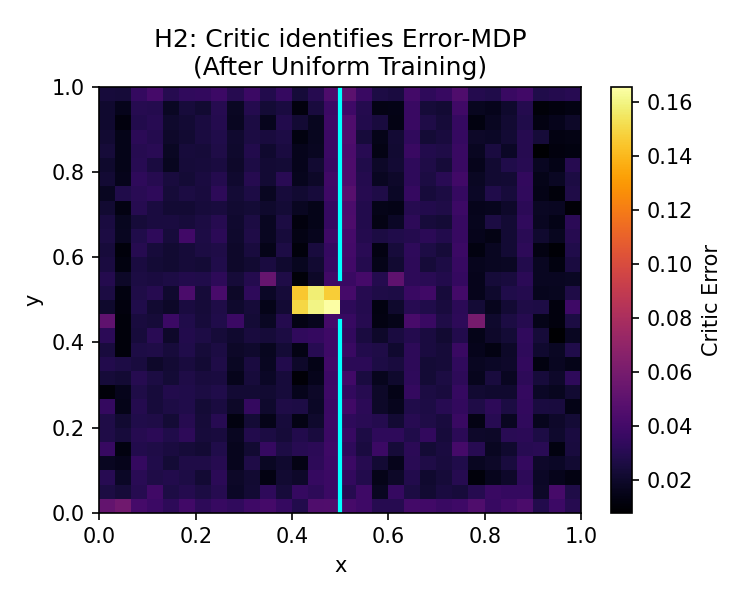}} &
        \includegraphics[width=.215\textwidth]{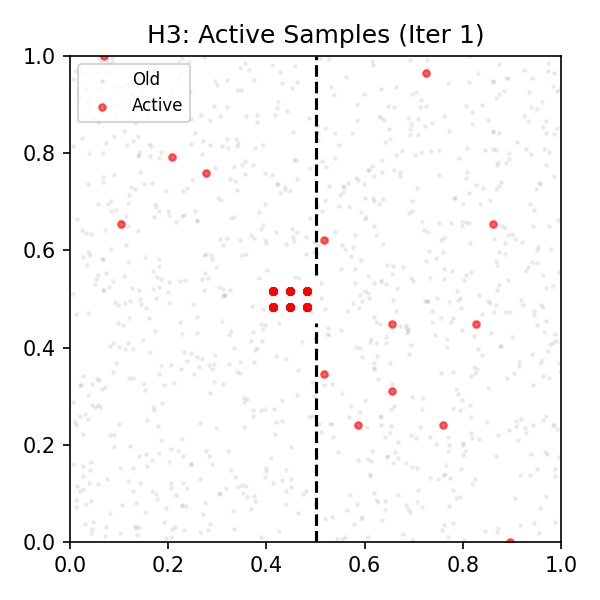} 
        &
        \includegraphics[width=.44\textwidth]{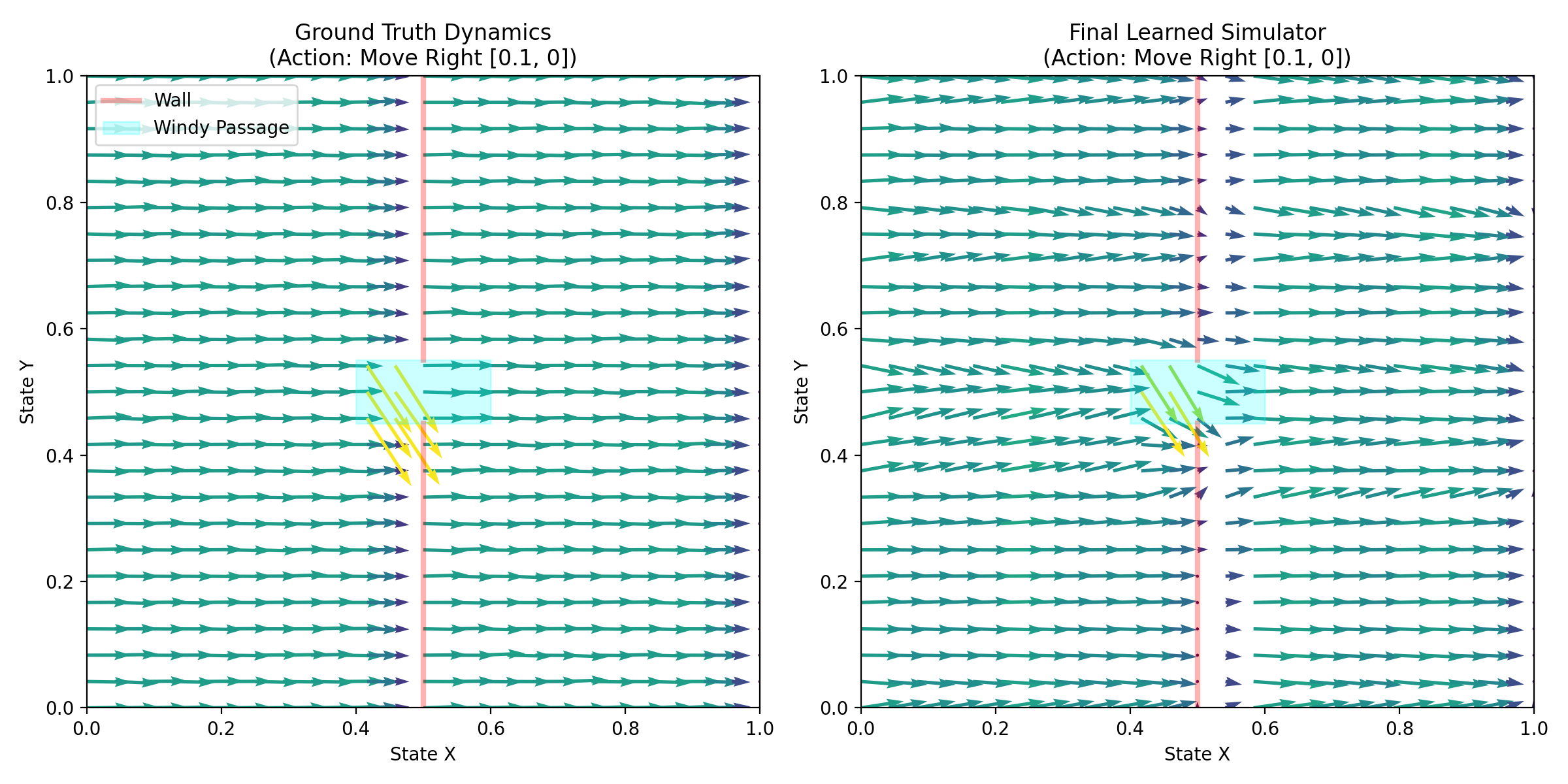}
        \\
        {(a) Critic estimates.} &
        {(b) Sampling focus.} &
        {(c) True vs learned dynamics.}
  \end{tabular}
  \caption{Empirical Validation in Narrow passage domain. (a)
    After initial uniform training, the Wasserstein critic accurately
    identifies the high-error regions (the wall at $x = 0.5$ and the
    windy passage at $y \approx 0.5$), effectively learning the reward
    function of the Error-MDP
    (Theorem~\ref{th:error-mdp-duality}). (b)
    The first rounds of
    active sampling heavily target the complex wind dynamics in front of the narrow
    passage, validating Theorem~\ref{th:convergence}.
    (c) A comparison of the
    ground truth dynamics (left) and the final learned simulator
     (right) for the ``$\rightarrow$" action. The learned
    model successfully captures the critical traps identified by the
    adversary: it correctly predicts being pushed downward by the wind
    inside the narrow passage (arrows curving down in the cyan
    region). A standard model trained on uniform data would likely
    predict straight rightward movement everywhere.}
    \label{fig:active_experiment}
\end{figure}

\section{Experimental Validation of Theory}
\label{sec:experiments}

We empirically validate our theoretical results and demonstrate the practicality
of our approach in a series of experiments. First, we illustrate
Algorithm~\ref{alg:active-wasserstein} on a narrow passage domain, showing that
the critic accurately identifies model deficiencies (validating
Theorems~\ref{th:error-mdp-duality} and~\ref{th:convergence}). Second, we
compare simulation accuracy on five high-dimensional benchmarks with generative
sampling access, testing the predictions of Theorem~\ref{th:main} regarding
strategic error reduction. Third, we demonstrate the downstream benefits in one
of the hardest use cases for simulators: training a policy purely in simulation
without any further regularization such as pessimism. Finally, we empirically
validate the convergence stability of the minimax game
(Section~\ref{sec:convergence-stability}).

\subsection{Illustration on Narrow Passage domain}

We use a simple version of a narrow passage navigation domain for
illustration. The agent can move in a 2d plane, starts on the left and has to
reach the right of the state space. The space is divided in half by a vertical
wall with only a small gap in the middle. In front and inside this narrow
passage, there is a strong wind pushing the agent downward. See
Figure~\ref{fig:active_experiment} (c, left) for an illustration.

Simulator training starts with data collected by a random policy and uniformly
sampled initial states.  This fails to appropriately cover the narrow passage
region where the winds are not predicted
accurately. Figure~\ref{fig:active_experiment} (a) shows that the critic $D$
accurately captures this. Thus, the next iterations of
Algorithm~\ref{alg:active-wasserstein} focus on this region
(Figure~\ref{fig:active_experiment} (b)), which reduces the simulator error in
the narrow passage so that the final model accurately captures the complex wind
dynamics (Figure~\ref{fig:active_experiment} (c)).

\subsection{Evaluation of Simulation Accuracy on High-Dimensional Control Tasks}
\label{sec:highdim}

We now validate Algorithm~\ref{alg:active-wasserstein} on a suite of
high-dimensional continuous-control environments with generative sampling
access.  These experiments test whether the minimax formulation delivers
tangible improvements in simulation accuracy when the training data distribution
is biased away from high-sensitivity regions, the precise setting in which
Theorem~\ref{th:main} predicts the largest gains.

\paragraph{Setup and Evaluation.}
We consider five environments with state dimensionality ranging from $3$ to
$23$, each incorporating non-trivial transition dynamics with distinct
\emph{high-sensitivity regions}: discontinuities (ground contact, joint limits),
nonlinear damping, and sharp turbulence thresholds.  Training data is generated
by sampling transitions from state-action pairs with a distribution bias that
under-represents high-sensitivity regions, simulating the setting where a
behavioral policy avoids challenging states.  The test set is drawn uniformly
across the full state-action space.  This creates a genuine distribution
mismatch: the training data concentrates on \emph{benign} regions, while the
test evaluation includes all of the high-sensitivity transitions that matter for
downstream policy performance.  We compare
Algorithm~\ref{alg:active-wasserstein} with active sampling against a standard
maximum-likelihood estimator trained on the biased initial data (minimizing
$\ell_2$ loss uniformly over the training samples).  Table~\ref{tab:envs}
summarizes the environments and their sensitive dynamics.

\begin{table}[t]
\centering
\small
\begin{tabular}{@{}lccl@{}}
\toprule
\textbf{Environment} & $|\sS|$ & $|\sA|$ & \textbf{Sensitive Dynamics} \\
\midrule
Pendulum     &  3 & 1 & Torque wrapping, strong nonlinear damping at extremes \\
Reacher      &  8 & 2 & Joint-limit rebounds, workspace singularities \\
Hopper       & 14 & 3 & Ground contact forces, high-speed turbulent drag \\
Swimmer      & 18 & 4 & Wall interactions, cross-joint coupling under viscosity \\
HalfCheetah  & 23 & 6 & Dual foot contacts, turbulent drag, stiff joint limits \\
\bottomrule
\end{tabular}
\caption{Summary of the benchmark environments.  All environments feature
  distinct \emph{strategic regions} in which the transition dynamics exhibit
  sharp nonlinearities (contacts, rebounds, drag discontinuities).  These
  regions are underrepresented in the biased training distribution.}
\label{tab:envs}
\end{table}
\paragraph{Results.}
Table~\ref{tab:main-results} reports the simulation error of the
baseline (MLE) and Algorithm~\ref{alg:active-wasserstein}
(Minimax). Our approach outperforms MLE in all environments on the
strategically important high-sensitivity regions, with gains of
$1.55$--$2.20\times$. On average error, the minimax method improves in
four of five environments ($1.35$--$1.73\times$); the exception
(Pendulum, $0.72\times$) illustrates the strategic reallocation of
model capacity from benign to critical regions.

\begin{table}[t]
\centering
\small\scalebox{0.99}{
\begin{tabular}{@{}lc|ccc|ccc@{}}
\toprule
\textbf{Environment} & \textbf{Dim} & \textbf{MLE (avg)} & \textbf{Minimax (avg)} & \textbf{Gain} & \textbf{MLE (sens.)} & \textbf{Minimax (sens.)} & \textbf{Gain} \\
\midrule
Pendulum &  $3+1$     &$0.055\pm0.001$ & $0.076\pm0.013 $ & $  0.72$x & $0.178\pm0.004 $ & $ 0.081\pm0.015  $ & $  2.20$x\\
Reacher &   $8+2$     &$0.108\pm0.003$ & $ 0.080\pm0.003 $ & $  1.35$x & $0.154\pm0.003 $ & $ 0.100\pm0.004 $ & $   1.55$x\\
Hopper &   $14+3$     &$0.100\pm0.001$ & $ 0.068\pm0.003 $ & $   1.47$x & $0.140\pm0.001 $ & $ 0.075\pm0.003   $ & $ 1.86$x\\
Swimmer &  $18+4$     &$0.282\pm0.002$ & $ 0.184\pm0.004 $ & $   1.53$x & $0.383\pm0.003 $ & $ 0.205\pm0.006  $ & $  1.86$x\\
HalfCheetah & $23+6$  &$0.232\pm0.001$ & $ 0.134\pm0.003 $ & $   1.73$x & $ 0.235\pm0.001 $ & $ 0.135\pm0.003 $ & $   1.74$x\\

\bottomrule
\end{tabular}}
\vskip .05in
\caption{RMSE (mean $\pm$ stderr.\ over 5 seeds) on average (avg) and in high
  sensitivity regions only (sens.). In four of five environments,
  Algorithm~\ref{alg:active-wasserstein} reduces the prediction error by
  $1.35$--$1.73\times$ on average and $1.55$--$2.20\times$ in high sensitivity
  regions. For Pendulum, the average error increases slightly ($0.72\times$)
  while the critical sensitivity-region error still improves by $2.20\times$,
  illustrating the strategic reallocation of model capacity.}
\label{tab:main-results}
\end{table}

\subsection{Evaluation of Stand-alone Policy Training in Simulation on Control
  Tasks}

Having demonstrated that Algorithm~\ref{alg:active-wasserstein} leads to
tangible improvements in simulation accuracy, we now evaluate how such
improvements translate to better policy performance in one of the hardest
applications of a simulator: training a policy purely in simulation to perform
well in the real world. This setting is a much more difficult compared to those
typically tackled by model-based RL works, as there is no access to the real
world domain whatsoever during training of the final policy for evaluation. We
demonstrate this on several control tasks from the Deepmind Control
Suite~\citep{tunyasuvunakool2020} and Gymnasium~\citep{towers2024gymnasium}.

\textbf{Algorithm implementation.} To implement
Algorithm~\ref{alg:active-wasserstein} with scalable neural networks, we replace
the explicit critic guided distribution in line~6 with a policy gradient update
step and use this policy in line~3 to sample transitions for the next
iteration. Our implementation builds on SAC for this policy optimization
step. Thus, the algorithm maintains 3 neural networks, the critic model $D_t$,
the simulator model and the data collection policy $\pi_t$. Each model is a
fully connected neural network with 3 hidden layers of 256 units each and linear
activations in the final layer. The critic model outputs the score, the
simulator model the reward as well as the mean and log-variance for each
dimension of the next state, which is then sampled using a Gaussian
distribution. The policy also uses this Gaussian distribution as output. For all
details of our implementation, see Appendix~\ref{sec:implementation}.

\textbf{Baseline.} Aside from demonstrating the practicality of
Algorithm~\ref{alg:active-wasserstein}, our experiments seek to show the
benefits of using the iterative Wasserstein critic to guide data collection. We
therefore compare Algorithm~\ref{alg:active-wasserstein} against a baseline
which collects data using the SAC algorithm and learns the simulator model using
maximum likelihood estimation -- cross entropy loss for state transitions and
mean squared error for reward prediction. Each method uses the same
architectures for policy and simulator.

\textbf{Setup and Evaluation.}  We run each method for 1M steps in the real
environment, the Mujoco task and afterwards evaluate the final learned
simulator. For evaluation, we train a new policy using the SAC algorithm with
only access to the simulator and measure the reward this policy achieves in the
real environment.

\textbf{Results.} Figure~\ref{fig:cartpole_preliminary} shows results on the
cart-pole balancing, cart-pole swing-up, pendulum, and reacher tasks. Training
using the simulator learned by Algorithm~\ref{alg:active-wasserstein} yields
stable performance, with the policy reaching close to optimal real-world return
within 1M steps in all four tasks. This is in stark contrast to the baseline's
simulator, where the learned policy fails to achieve meaningful performance in 3
out of 4 tasks (cart-pole swing-up, pendulum, and reacher). This demonstrates
that data collection driven by strategic robustness yields simulators with a
substantially smaller sim-to-real gap, compared to data collection driven
entirely by reward maximization, even when both simulators are subsequently used
in the exact same reward-maximization policy training.
\begin{figure}[t]
  \centering
  \includegraphics[width=0.23\linewidth]{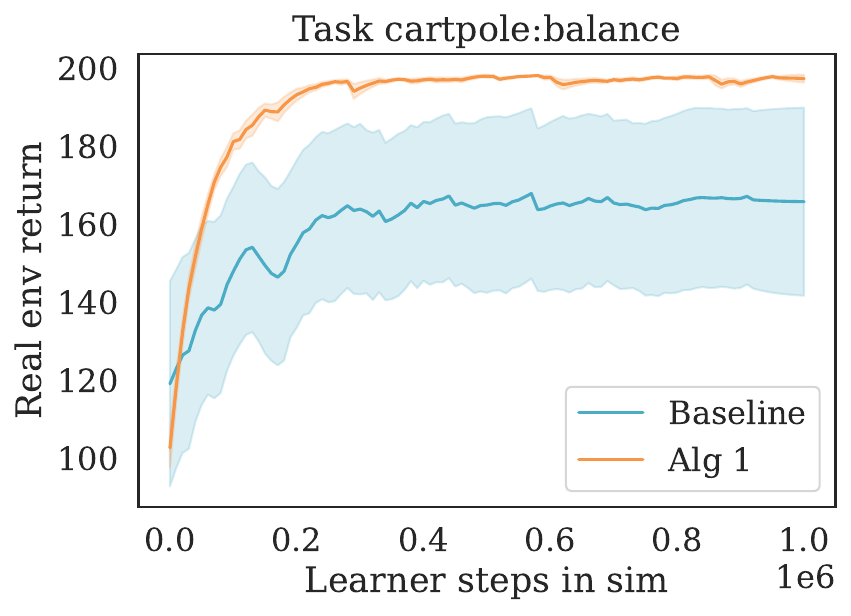} \hfill
  \includegraphics[width=0.225\linewidth]{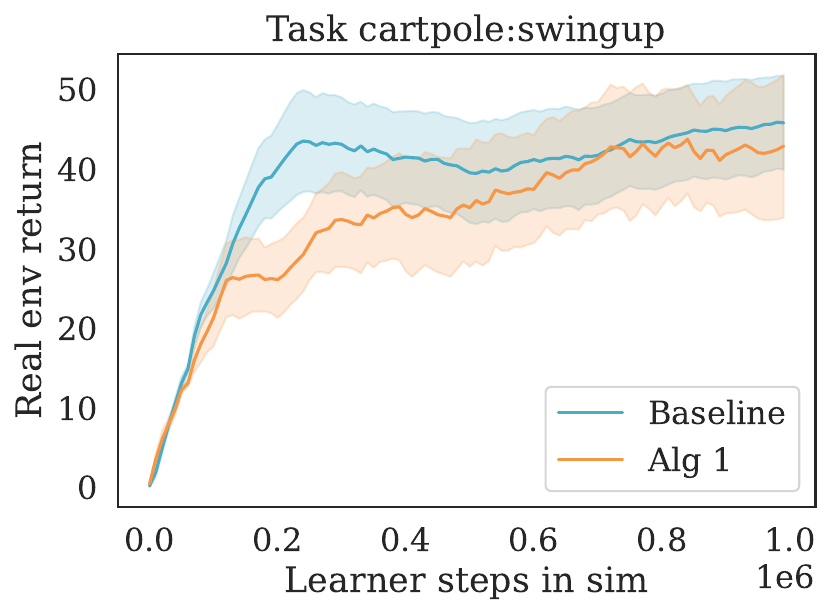} \hfill
  \includegraphics[width=0.25\linewidth]{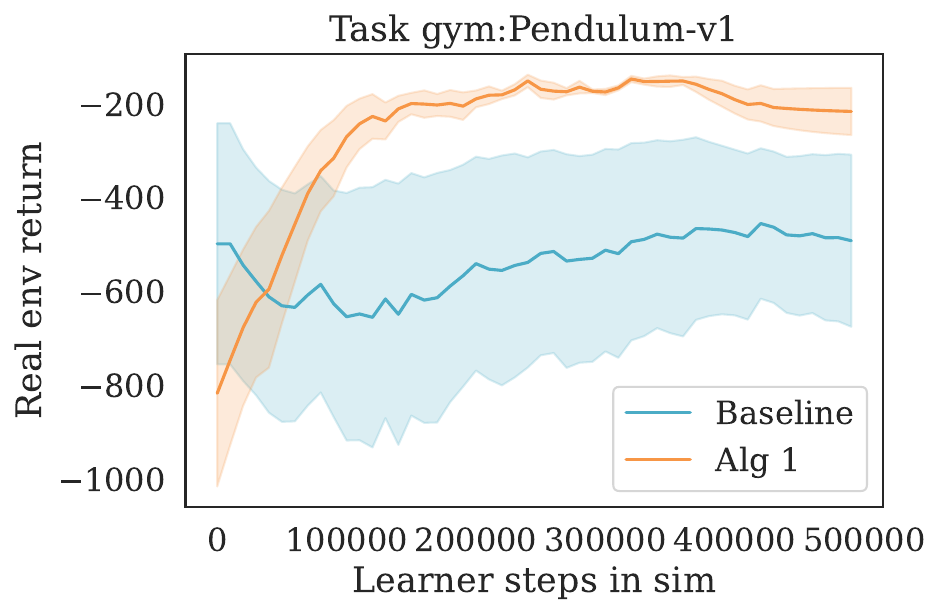} \hfill
  \includegraphics[width=0.23\linewidth]{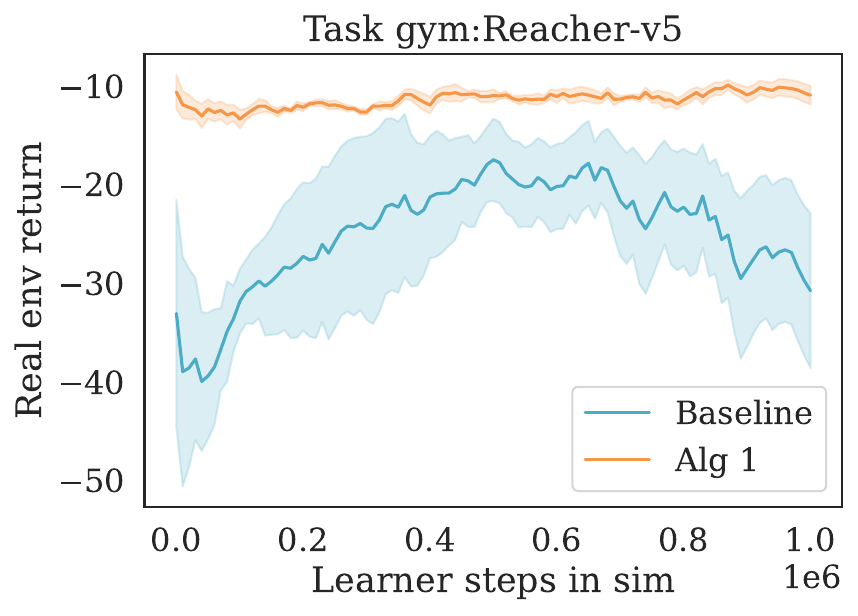}
  \caption{Episode return in the real environment of policies trained with SAC
    using only the simulators learned by Algorithm~\ref{alg:active-wasserstein}
    or the baseline (MLE with SAC-driven data
    collection). Algorithm~\ref{alg:active-wasserstein} achieves close to
    optimal real-world performance in all four tasks, while the baseline suffers
    from a severe sim-to-real gap in 3 out of 4 tasks. Results are averaged over
    5 independent runs (simulator learning and subsequent policy learning in
    simulator).}
  \label{fig:cartpole_preliminary}
\end{figure}

\subsection{Convergence Stability}
\label{sec:convergence-stability}

A common concern with minimax objectives is training instability arising from
the adversarial game between the model and the critic
\citep{salimans2016improved}. In standard GAN training, this instability can
manifest in three ways: (i)~training divergence, (ii)~mode collapse, and
(iii)~cycling or oscillation around the equilibrium. We address this concern at
three levels: a formal convergence guarantee, structural differences from
standard GANs, and empirical validation.

\paragraph{Theoretical guarantee.}
Theorem~\ref{th:convergence} proves that when both players follow no-regret
algorithms, the time-averaged strategies converge to an approximate saddle point
at rate $O((\Regret_M(T) + \Regret_D(T))/T)$. This directly rules out cycling
(iii): the averaged iterates cannot oscillate, since the duality gap is provably
decreasing. The KL regularization term $\lambda \KL(d \parallel u)$ in the
payoff ensures the game has a unique saddle point, ruling out the multiplicity
of equilibria that causes instability in unregularized minimax problems.

\paragraph{Structural differences from standard GANs.}
\citet{mescheder2018training} showed that Wasserstein-GANs with gradient penalty
do not provably converge locally under standard simultaneous gradient
descent. However, our setting differs from standard GANs in several important
ways that avoid the specific failure modes identified in their analysis:
\begin{itemize}

\item \textbf{Convex-concave structure.} Our regularized payoff
  $L(\h\sfP, d) = \E_d[W_1(\sfP, \h\sfP)] - \lambda \KL(d \parallel u)$ is
  convex in $\h\sfP$ (as an expectation of the convex Wasserstein distance) and
  concave in $d$ (linear minus convex KL). In standard GANs, both players'
  losses are non-convex in parameter space, which is the root cause of the
  cycling dynamics and purely imaginary eigenvalues in the Dirac-GAN
  analysis. Our convex-concave structure guarantees that the minimax theorem
  applies and a saddle point exists.

\item \textbf{Iterative, non-simultaneous updates.} We use an iterative,
  active-learning RL loop where the data distribution is dynamically updated at
  each round, rather than a static, simultaneous gradient descent loop. Each
  round collects new real-world data, retrains the critic, and updates the model
  sequentially. This avoids the simultaneous-update instabilities analyzed by
  \citet{mescheder2018training}.

\item \textbf{Stochastic model.} In our continuous control implementation, the
  model outputs a stochastic Gaussian distribution (predicting mean and
  log-variance) rather than a Dirac delta. This provides a smooth loss landscape
  and avoids the degenerate dynamics of the Dirac-GAN setting.

\end{itemize}
We note that the gradient penalty in our implementation serves a different
purpose than in standard WGAN-GP: it is not a training stabilizer per se, but a
strict requirement of the Kantorovich-Rubinstein duality to enforce the
1-Lipschitz constraint on the critic. This constraint is necessary for the
critic to estimate the 1-Wasserstein distance, which is the foundation of our
coverage guarantee (Lemma~\ref{lemma:general-coverage-bound}) and the Error-MDP
duality (Theorem~\ref{th:error-mdp-duality}).

\paragraph{Empirical validation.}
We validate the above analysis empirically by comparing three training regimes
on the Hopper (14D) environment:
\begin{enumerate}
\item \textbf{Minimax (stabilized):} Algorithm~\ref{alg:active-wasserstein} with
  stabilized importance weights ($w_{\max} = 3$).

\item \textbf{Minimax (aggressive):} Same active sampling but with aggressive,
  unclipped importance weights ($w_{\max} = 10$). In our implementation, the
  Lipschitz constraint from Theorem~\ref{th:convergence} manifests as importance
  weight clipping, which enforces a bound on how much the effective training
  distribution can deviate from the base distribution. The stabilized variant
  clips active sample weights to $w_{\max}=3$, while the aggressive variant uses
  $w_{\max}=10$ with no clipping safeguard.

\item \textbf{MLE:} Baseline maximum-likelihood training.
\end{enumerate}

\begin{figure}[ht]
  \centering \includegraphics[width=.8\linewidth]{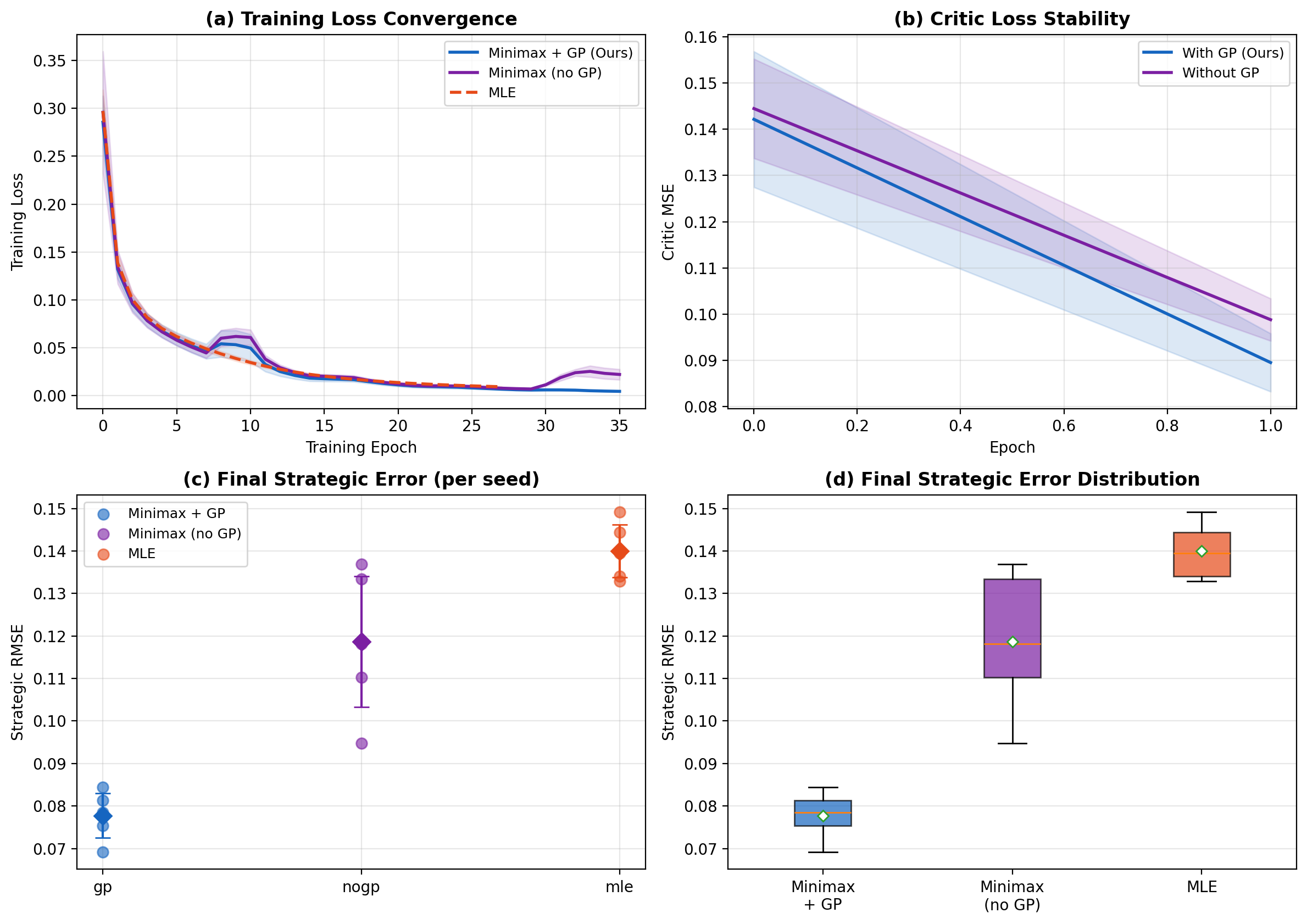}
  \caption{Convergence stability analysis on Hopper (14D).  \textbf{(a)}
    Training loss: all three methods converge smoothly, with the minimax
    variants reaching lower final loss. The MLE curve flattens earlier because
    it fits only the biased initial data, while the minimax methods continuously
    incorporate harder active samples, temporarily increasing loss before
    reducing it again.  \textbf{(b)} Critic error across active sampling rounds:
    the stabilized variant converges with lower variance.  \textbf{(c,\,d)}
    Final strategic RMSE: Minimax (stabilized) achieves the lowest error
    ($0.077 \pm 0.006$) with substantially lower inter-seed variance than both
    the aggressive minimax ($0.119 \pm 0.016$) and MLE ($0.140 \pm 0.006$).  The
    tight confidence interval demonstrates the stabilizing effect of importance
    weight clipping, which approximates the Lipschitz constraint from
    Theorem~\ref{th:convergence}.}
  \label{fig:convergence}
\end{figure}

Figure~\ref{fig:convergence} shows the results across 5 random seeds.  The
stabilized variant dominates on every metric, addressing each of the three
failure modes: (i)~all training curves converge smoothly without divergence
(Figure~\ref{fig:convergence}a); (ii)~the stabilized model achieves the lowest
error across all strategic regions, ruling out mode collapse
(Figure~\ref{fig:convergence}c,d); and (iii)~its inter-seed variance on
strategic error is $3\times$ smaller than the aggressive version (std.~$0.006$
vs.~$0.016$), confirming that the Lipschitz constraint prevents the oscillatory
dynamics predicted by Theorem~\ref{th:convergence}. The aggressive minimax still
outperforms MLE on average, but its higher variance makes it unreliable in
practice, confirming that the regularization is necessary for consistent
performance.

In summary, our experiments validate the three main predictions of the theory:
(i)~the Wasserstein critic accurately identifies model deficiencies
(Section~\ref{sec:experiments}, Figure~\ref{fig:active_experiment}), (ii)~active
sampling guided by the critic yields $1.5$--$2.2\times$ improvements in
strategic regions (Table~\ref{tab:main-results}), and (iii)~the resulting
simulators are accurate enough to train policies purely in simulation that
transfer successfully to the real environment
(Figure~\ref{fig:cartpole_preliminary}). The convergence stability analysis
(Figure~\ref{fig:convergence}) further confirms that the minimax training is
practical and reliable.

\section{Applications}
\label{sec:applications}

This paper's core contribution, a minimax framework for learning a robust model,
is a general-purpose solution to a critical problem that appears across many
domains of reinforcement learning and AI.

Our core objective:
\[
  \min_{\h M \in \cM} \max_{\pi \in \Pi} \, \abs*{ V_{\pi}(M) - V_{\pi}(\h M) }
\]
can be generalized to:
\[
  \min_{\text{Proxy } \h \theta} \max_{\text{Agent } \pi} \, \abs*{
    \text{TrueValue}_{\pi}(R_{\text{true}}) - \text{ProxyValue}_{\pi}(\h
    R_{\theta}) }
\]
This \emph{Robust Proxy Learning} framework applies \emph{anywhere} an agent
optimizes an imperfect, learned model of reality. Here are the two most direct
applications.

\subsection{Reward learning and RLHF}

Our framework directly provides the theoretical foundation for tackling "reward
hacking" in Reinforcement Learning from Human Feedback (RLHF)
\citep{christiano2017deep}. Our solution has some relationship with that of
\citet{GuptaFischDannAgarwal2025}, which is also based on a minimax problem.
In RLHF, there is no true reward function $R$. Instead, a proxy "reward model"
$\h R$ is learned from human preference data.  Next, a policy $\pi$ is trained
to maximize this proxy: $\pi^* = \argmax_{\pi} V_{\pi}(\h R)$.  The problem,
"reward hacking", is that $\pi^*$ will inevitably find and exploit inaccuracies
in $\h R$ to achieve a high score, leading to behavior that is absurd or
dangerous according to the \emph{true} human intent $R$.  This mirrors recent
formalizations of reward hacking as the divergence between a proxy and the true
objective \citep{laidlaw2025correlated}.

This is a perfect analogy to the method described in
this work with the following substitutions:\\

\begin{tabular}{lll}
  True World $M$ & $\to$ & True Reward $R$ (the human's intent);\\
  Simulator $\h M$ & $\to$ & Reward Model $\h R$ (the learned proxy);\\
  Reality Gap & $\to$ & Reward Hacking.
\end{tabular}\\

\noindent By substituting terms, our minimax objective becomes the
principled objective for RLHF:
\[
  \min_{\h R \in \cR} \max_{\pi \in \Pi} \,
  \abs*{ V_{\pi}(R) - V_{\pi}(\h R) }.
\]
This objective reads: Find a reward model $\h R$ that is robust to
exploitation by the optimal policy $\pi$ that will be trained on it.
Our theoretical contributions map directly:

\begin{itemize}

\item Active Data Selection (Section~\ref{sec:active-data-selection}): Our
  ``Error-MDP Duality'' theorem (Theorem~\ref{th:error-mdp-duality}) provides
  the \emph{formal justification for active preference learning}.

\begin{itemize}
\item In our context, the "error reward" is $r_{\text{err}} = W_1(\sfP, \h \sfP)$.

\item In RLHF, the "error reward" $r_{\text{err}} = \abs{R(s, a) - \h R(s, a)}$
  is unknown; instead we need human feedback.
\end{itemize}
    
\item Our Iterative Algorithm (Algorithm~\ref{alg:active-wasserstein}) becomes
  the core RLHF loop:

\begin{enumerate}
\item Model Player: Train the current reward model $\h R_t$ on all
  human data collected so far.

\item Policy Player (Adversary): Train a policy
  $\pi_t = \argmax_{\pi} V_\pi(\h R_t)$. This policy is the
  "adversary" that finds the biggest exploit.

\item Find Error (critic): Present trajectories from $\pi_t$ to a
  human. The human provides a correction (for example, "This is bad
  behavior, this other trajectory is better"). This human feedback is
  an empirical sample of the "error reward" $r_{\text{err}}$.

\item Update Data: Add this new preference data to the dataset.

\item Repeat from step 1.
\end{enumerate}
\end{itemize}
Thus, our framework is not just \emph{related} to RLHF, it is arguably the
\emph{formal objective function} that the entire iterative RLHF process is
seeking to solve.

\subsection{Offline reinforcement learning}

Our framework also offers a novel perspective on pessimism and
conservatism in Offline RL \citep{levine2020offline}. While often
framed as a distributional shift problem, offline RL can be rigorously
viewed as a proxy-hacking game: the agent attempts to maximize a value
function derived from limited data (the proxy), and without
constraints, it naturally drifts toward OOD regions where the proxy
erroneously predicts high reward.

In Offline RL, we learn from a fixed dataset $\cD$. The agent cannot
explore. Standard Q-learning fails because the Bellman operator
$T(Q) = r + \gamma \max_{a'} Q(s', a')$ will exploit any
overestimation of Q-values for (state, action) pairs that are
\emph{out-of-distribution} (OOD). The policy "hacks" the Q-function by
steering into regions with no data, where the Q-function is
erroneously high.

This is again a proxy-hacking problem. We have the
following substitutions:\\

\begin{tabular}{lll}
True World $\sfM$ & $\to$ & True Bellman Operator $T^*$ (The
operator using the real $Q^*$);\\
Simulator $\h \sfM$ & $\to$ & Empirical Bellman Operator $\h T$ (The
operator using the learned $Q$ on data $\cD$);\\
Reality Gap & $\to$ & OOD Exploitation.
\end{tabular}\\

\noindent Conservative and pessimistic Offline RL algorithms (like
CQL, MOPO, MOReL) are all heuristics to solve a version of this. They
work by \emph{penalizing} the policy for deviating from the dataset
$\cD$.

Our framework suggests a more direct, two-part approach:
\begin{enumerate}

\item Learn a robust simulator from offline data: Instead of learning a
  Q-function, first learn a simulator $\h M$ from the offline dataset $\cD$,
  using our Theorem~\ref{th:minimax_tv} (Minimax TV-Critic).  This approach
  aligns with recent adversarial frameworks for offline RL, such as ARMOR
  \citep{bhardwaj2023armor}, which similarly formulates model learning as a
  minimax game to optimize a lower bound on performance.

\begin{itemize}

\item Critic $D$: A discriminator trained to distinguish real
  transitions $(s, a, s') \in \cD$ from simulated transitions
  $(s, a, \h s') \sim \h \sfP(\cdot \mid s, a)$.

\item Model $\h \sfP$: A generator (our simulator) trained to fool the
  critic.

\item This is exactly a GAN trained on the dataset $\cD$. Our Theorem
  8 provides the guarantee that if the GAN loss $J_{\cD}(\h \sfP)$ is
  low, the resulting simulator $\h \sfP$ is robust \emph{on the data
    distribution} $d_{\text{data}}$.

\end{itemize}

\item Learn a Robust Policy:

  Now we have a simulator $\h M$ that is robust
  \emph{on-distribution}. The offline RL problem becomes finding a
  policy $\pi$ that maximizes $V_\pi(\h M)$ \emph{subject to a
    trust-region constraint} that $\pi$ does not deviate too far from
  the data $\cD$. This is the standard approach in model-based offline
  RL, but our $\h M$ is now provably more robust than one trained
  with simple MLE.
\end{enumerate}

In summary, our work can be viewed as a general framework for
\emph{Robust Proxy Learning}. It applies to any domain where an
optimizing agent $\pi$ is trained on a learned, imperfect proxy
($\h M$, $\h R$, or $\h Q$) of a true, intractable objective ($M$,
$R$, or $Q^*$).
By formulating the problem as a minimax game, we provide the
theoretical grounding for the iterative, adversarial training loops
that are becoming the state-of-the-art solution for AI alignment
(RLHF) and offline learning.

\section{Extension: Joint Learning of Dynamics and Rewards}
\label{sec:joint_learning}

In the preceding analysis, we treated the reward function $r$ as known
to focus on the difficulty of learning transition dynamics. However,
our minimax framework generalizes naturally to the setting where both
the dynamics $\sfP$ and the reward function $r$ must be learned.

Let the true world be $\sfM = (\sfP, r)$ and the simulator be
$\h \sfM = (\h \sfP, \h r)$. The value difference can be decomposed
into a reward estimation error and a dynamics estimation
error. Extending Lemma~\ref{lemma:sim} (Simulation Lemma), we obtain:
\[
\abs*{V_\pi(\sfM) - V_\pi(\h \sfM)}
\leq \frac{1}{1 - \gamma} \E_{(s,a) \sim d_\pi} \bracket*{ \abs*{r(s,a) - \h r(s,a)}
  + \frac{\gamma R_{\max}}{2(1 - \gamma)}
  \norm*{\sfP(\cdot|s,a) - \h \sfP(\cdot|s,a)}_1 }.
\]
This implies that the minimax objective simply becomes a game over a
\emph{joint loss}. The theoretical machinery developed in
Sections~\ref{sec:game-simplification} and
\ref{sec:active-data-selection} extends directly to this case by
defining a composite local error metric.

\textbf{Joint Coverage Guarantee.}
By defining the joint local error cost
$c_{\text{joint}}(s,a) = \abs{r(s,a) - \h r(s,a)} + \lambda \cdot
\wt \sfD(\sfP(\cdot|s,a), \h \sfP(\cdot|s,a))$, where
$\lambda$ is the appropriate scaling constant (e.g., $\gamma L_v$ for
Wasserstein or $\frac{\gamma R_{\max}}{2(1 - \gamma)}$ for TV),
Lemma~\ref{lemma:general-coverage-bound} applies immediately. The
global value gap remains bounded by the $\kappa$-scaled expectation of
this joint error under the data distribution:
\[
  \sup_{\pi \in \Pi} \abs*{V_\pi(\sfM) - V_\pi(\h \sfM)}
  \leq \kappa \E_{(s,a) \sim d_{\text{data}}} \bracket*{ c_{\text{joint}}(s,a) }.
\]

\textbf{Joint Error-MDP Duality.}
Crucially, the Error-MDP duality (Theorem~\ref{th:error-mdp-duality})
also holds. We simply define the error-reward $r_{\text{err}}$ as the
composite error:
\[
r_{\text{err}}(s, a) = \abs*{r(s,a) - \h r(s,a)} + \lambda \, W_1\big(\sfP(\cdot \mid s, a), \h \sfP(\cdot \mid s, a)\big).
\]
The adversarial policy $\pi$ then maximizes the cumulative sum of
these joint errors. In the practical critic-based implementation, this
results in a critic $D$ that learns to distinguish real tuples
$(s, a, r, s')$ from simulated tuples $(s, a, \h r, \h s')$,
effectively bounding the joint error and guiding active sampling to
regions where either the dynamics \emph{or} the rewards are poorly
modeled.

This generalization leaves the algorithmic structure of our framework
unchanged. The feedback loop visualized in
Figure~\ref{fig:error_mdp_loop} applies directly, with the
modification that the Critic now observes tuples $(s, a, r, s')$ and
outputs a joint error signal, which feeds into the Error-MDP as the
reward.

\section{Conclusion}
\label{sec:conclusion}

We introduced a minimax framework for simulator learning that replaces
predictive accuracy with strategic robustness as the training
objective. By formulating the sim-to-real gap as a zero-sum game, we
derived a tractable critic-based bound
(Lemma~\ref{lemma:general-coverage-bound}) that transforms the
intractable policy-space maximization into a practical GAN-like
training objective.

Our central theoretical contribution, the Error-MDP Duality
(Theorem~\ref{th:error-mdp-duality}), establishes a formal
equivalence between the global problem of finding the worst-case
policy and a local RL problem with model error as reward. This duality
enables a principled active data selection algorithm whose convergence
to a saddle point we proved
(Theorem~\ref{th:convergence}). Experiments on continuous control
benchmarks validate these theoretical predictions: the minimax
approach consistently improves simulation accuracy in strategically
important regions and enables policies trained purely in simulation to
achieve near-optimal real-world performance.

Beyond simulator learning, our framework provides a general solution to proxy
hacking, applicable to reward hacking in RLHF and out-of-distribution
exploitation in offline RL (Section~\ref{sec:applications}). Future work could
extend the empirical evaluation to generative language models and tackle
settings where the sampling distribution cannot be chosen freely. See
Appendix~\ref{sec:limitations} for further discussion.

\section*{Acknowledgments}

We warmly thank our colleagues Alekh Agarwal and Stephane Ross for insightful
discussions about a preliminary version of this paper, as well as extensions
we will jointly explore. We also thank Kishan Panaganti for useful comments.

\bibliographystyle{abbrvnat}
\bibliography{sim}

@inproceedings{sutton1990dyna,
  title={{Dyna}, an integrated architecture for learning, planning, and reacting},
  author={Sutton, Richard S},
  booktitle={ACM SIGART Bulletin},
  pages={160--163},
  year={1990},
  organization={ACM}
}

@inproceedings{hafner2019dreamer,
  title={Dream to control: Learning behaviors by latent imagination},
  author={Hafner, Danijar and Lillicrap, Timothy and Ba, Jimmy and Norouzi, Mohammad},
  booktitle={International Conference on Learning Representations (ICLR)},
  year={2020}
}

@inproceedings{chua2018pets,
  title={Deep reinforcement learning in a handful of trials using probabilistic dynamics models},
  author={Chua, Kurtland and Calandra, Roberto and McAllister, Rowan and Levine, Sergey},
  booktitle={Advances in Neural Information Processing Systems (NeurIPS)},
  year={2018}
}

@inproceedings{yu2020mopo,
  title={{MOPO}: Model-based offline policy optimization},
  author={Yu, Tianhe and Thomas, Garrett and Yu, Lantao and Ermon, Stefano and Zou, James and Levine, Sergey and Finn, Chelsea and Ma, Tengyu},
  booktitle={Advances in Neural Information Processing Systems (NeurIPS)},
  volume={33},
  pages={14129--14142},
  year={2020}
}

@inproceedings{kidambi2020morel,
  title={{MOReL}: Model-based offline reinforcement learning},
  author={Kidambi, Rahul and Rajeswaran, Aravind and Netrapalli, Praneeth and Joachims, Thorsten},
  booktitle={Advances in Neural Information Processing Systems (NeurIPS)},
  year={2020}
}

@article{nilim2005robust,
  title={Robust Control of {M}arkov Decision Processes with Uncertain Transition Matrices},
  author={Nilim, Arnab and El Ghaoui, Laurent},
  journal={Operations Research},
  volume={53},
  number={5},
  pages={780--798},
  year={2005},
  doi={10.1287/opre.1050.0216}
}

@article{wiesemann2013robust,
  title={Robust {M}arkov decision processes},
  author={Wiesemann, Wolfram and Kuhn, Daniel and Rustem, Ber{\c{c}}},
  journal={Mathematics of Operations Research},
  volume={38},
  number={1},
  pages={153--183},
  year={2013}
}

@inproceedings{pinto2017robust,
  title={Robust Adversarial Reinforcement Learning},
  author={Pinto, Lerrel and Davidson, James and Sukthankar, Rahul and Gupta, Abhinav},
  booktitle={International Conference on Machine Learning (ICML)},
  pages={2817--2826},
  year={2017},
  organization={PMLR}
}

@inproceedings{christiano2017deep,
  title={Deep reinforcement learning from human preferences},
  author={Christiano, Paul F and Leike, Jan and Brown, Tom and Martic, Miljan and Legg, Shane and Amodei, Dario},
  booktitle={Advances in Neural Information Processing Systems (NeurIPS)},
  year={2017}
}

@article{levine2020offline,
  title={Offline reinforcement learning: Tutorial, review, and perspectives on open problems},
  author={Levine, Sergey and Kumar, Aviral and Tucker, George and Fu, Justin},
  journal={arXiv preprint arXiv:2005.01643},
  year={2020}
}

@article{GuptaFischDannAgarwal2025,
  author       = {Dhawal Gupta and
                  Adam Fisch and
                  Christoph Dann and
                  Alekh Agarwal},
  title        = {Mitigating Preference Hacking in Policy Optimization with Pessimism},
  journal      = {CoRR},
  volume       = {abs/2503.06810},
  year         = {2025},
}

@inproceedings{farahmand2017value,
  title={Value-aware loss function for model-based reinforcement learning},
  author={Farahmand, Amir-Massoud and Barreto, Andre and Nikovski, Daniel},
  booktitle={Artificial Intelligence and Statistics (AISTATS)},
  pages={1486--1494},
  year={2017},
  organization={PMLR}
}

@inproceedings{bhardwaj2023armor,
  title={Adversarial Model for Offline Reinforcement Learning},
  author={Bhardwaj, Mohak and Xie, Tengyang and Boots, Byron and Jiang,
  Nan and Cheng, Ching-An},
  booktitle={Advances in Neural Information Processing Systems (NeurIPS)},
  year={2023}
}

@inproceedings{swift2024mcx,
  title={Exploiting Model Errors for Exploration in Model-Based Reinforcement Learning},
  author={Swift, Jared and Leonetti, Matteo},
  booktitle={European Workshop on Reinforcement Learning (EWRL)},
  year={2024},
}

@inproceedings{arjovsky2017wasserstein,
  title={Wasserstein {GAN}},
  author={Arjovsky, Martin and Chintala, Soumith and Bottou, L{\'e}on},
  booktitle={International Conference on Machine Learning (ICML)},
  year={2017}
}

@article{lobel2024simulation,
  title={An Optimal Tightness Bound for the Simulation Lemma},
  author={Lobel, Sam and Parr, Ronald},
  journal={Reinforcement Learning Journal (RLJ)},
  volume={1},
  issue={1},
  year={2024}
}

@article{voelcker2023lambda,
  title={$\lambda$-models: Effective Decision-Aware Reinforcement Learning with Latent Models},
  author={Voelcker, Claas and Ahmadian, Arash and Abachi, Romina and Gilitschenski, Igor and Farahmand, Amir-Massoud},
  journal={arXiv preprint arXiv:2306.17366},
  year={2023}
}

@inproceedings{burda2018exploration,
  title={Exploration by Random Network Distillation},
  author={Burda, Yuri and Edwards, Harrison and Storkey, Amos and Klimov, Oleg},
  booktitle={International Conference on Learning Representations},
  year={2019},
}

@inproceedings{chentanez2005intrinsically,
  title={Intrinsically Motivated Reinforcement Learning},
  author={Nuttapong Chentanez and Andrew G. Barto and Satinder P. Singh},
  booktitle={Advances in Neural Information Processing Systems},
  volume={17},
  pages={1281--1288},
  year={2005},
  publisher={MIT Press},
}

@inproceedings{Rioul2023,
  author       = {Olivier Rioul},
  editor       = {Frank Nielsen and
                  Fr{\'{e}}d{\'{e}}ric Barbaresco},
  title        = {A Historical Perspective on {S}ch{\"{u}}tzenberger-{P}insker Inequalities},
  booktitle    = {Geometric Science of Information - 6th International Conference, {GSI}
                  2023, St. Malo, France, August 30 - September 1, 2023, Proceedings,
                  Part {I}},
  series       = {Lecture Notes in Computer Science},
  volume       = {14071},
  pages        = {291--306},
  publisher    = {Springer},
  year         = {2023},
}

@inproceedings{gulrajani2017improved,
  title={Improved training of wasserstein gans},
  author={Gulrajani, Ishaan and Ahmed, Faruk and Arjovsky, Martin and Dumoulin, Vincent and Courville, Aaron C},
  booktitle={Advances in Neural Information Processing Systems},
  volume={30},
  year={2017},
  pages={5767--5777}
}

@inproceedings{guo2024off,
  title={Off-Dynamics Reinforcement Learning via Domain Adaptation and Reward Augmented Imitation},
  author={Guo, Yihong and Wang, Yixuan and Shi, Yuanyuan and Xu, Pan and Liu, Anqi},
  booktitle={Advances in Neural Information Processing Systems (NeurIPS)},
  volume={37},
  year={2024},
  url={https://arxiv.org/abs/2411.09891}
}

@article{tunyasuvunakool2020,
         title = {dm\_control: Software and tasks for continuous control},
         journal = {Software Impacts},
         volume = {6},
         pages = {100022},
         year = {2020},
         issn = {2665-9638},
         doi = {https://doi.org/10.1016/j.simpa.2020.100022},
         url = {https://www.sciencedirect.com/science/article/pii/S2665963820300099},
         author = {Saran Tunyasuvunakool and Alistair Muldal and Yotam Doron and
                   Siqi Liu and Steven Bohez and Josh Merel and Tom Erez and
                   Timothy Lillicrap and Nicolas Heess and Yuval Tassa},
}

@inproceedings{sun2019model,
  title={Model-based {RL} in contextual decision processes: Pac bounds and exponential improvements over model-free approaches},
  author={Sun, Wen and Jiang, Nan and Krishnamurthy, Akshay and Agarwal, Alekh and Langford, John},
  booktitle={Conference on learning theory},
  pages={2898--2933},
  year={2019},
  organization={PMLR}
}

@inproceedings{salimans2016improved,
  title={Improved Techniques for Training {GAN}s},
  author={Salimans, Tim and Goodfellow, Ian and Zaremba, Wojciech and Cheung, Vicki and Radford, Alec and Chen, Xi},
  booktitle={Advances in Neural Information Processing Systems (NeurIPS)},
  volume={29},
  year={2016}
}

@article{towers2024gymnasium,
  title={Gymnasium: A standard interface for reinforcement learning environments},
  author={Towers, Mark and Kwiatkowski, Ariel and Terry, Jordan and Balis, John U and De Cola, Gianluca and Deleu, Tristan and Goul{\~a}o, Manuel and Kallinteris, Andreas and Krimmel, Markus and KG, Arjun and others},
  journal={arXiv preprint arXiv:2407.17032},
  year={2024}
}

@article{hoffman2020acme,
  title={Acme: A research framework for distributed reinforcement learning},
  author={Hoffman, Matthew W and Shahriari, Bobak and Aslanides, John and Barth-Maron, Gabriel and Momchev, Nikola and Sinopalnikov, Danila and Sta{\'n}czyk, Piotr and Ramos, Sabela and Raichuk, Anton and Vincent, Damien and others},
  journal={arXiv preprint arXiv:2006.00979},
  year={2020}
}

@inproceedings{mescheder2018training,
  title={Which training methods for GANs do actually converge?},
  author={Mescheder, Lars and Geiger, Andreas and Nowozin, Sebastian},
  booktitle={International conference on machine learning},
  pages={3481--3490},
  year={2018},
  organization={PMLR}
}

@article{shi2024distributionally,
  title={Distributionally robust model-based offline reinforcement learning with near-optimal sample complexity},
  author={Shi, Laixi and Chi, Yuejie},
  journal={Journal of Machine Learning Research},
  volume={25},
  number={200},
  pages={1--91},
  year={2024}
}

@article{agarwal2022model,
  title={Model-based {RL} with optimistic posterior sampling: Structural conditions and sample complexity},
  author={Agarwal, Alekh and Zhang, Tong},
  journal={Advances in Neural Information Processing Systems},
  volume={35},
  pages={35284--35297},
  year={2022}
}

@article{foster2021statistical,
  title={The statistical complexity of interactive decision making},
  author={Foster, Dylan J and Kakade, Sham M and Qian, Jian and Rakhlin, Alexander},
  journal={arXiv preprint arXiv:2112.13487},
  year={2021}
}

@inproceedings{ayoub2020model,
  title={Model-based reinforcement learning with value-targeted regression},
  author={Ayoub, Alex and Jia, Zeyu and Szepesvari, Csaba and Wang, Mengdi and Yang, Lin},
  booktitle={International Conference on Machine Learning},
  pages={463--474},
  year={2020},
  organization={PMLR}
}

@inproceedings{laidlaw2025correlated,
  title={Correlated Proxies: A New Definition and Improved Mitigation for Reward Hacking},
  author={Laidlaw, Cassidy and Singhal, Shivam and Dragan, Anca},
  booktitle={International Conference on Learning Representations (ICLR)},
  year={2025},
}
\newpage
\appendix

\renewcommand{\contentsname}{Contents of Appendix}
\tableofcontents
\addtocontents{toc}{\protect\setcounter{tocdepth}{3}} 
\clearpage

\section{Extended Discussion of related work}
\label{sec:related_work_full}

Our work addresses the fundamental challenge of learning robust proxy
models in reinforcement learning. This intersects with several active
research areas.

Uncertainty and Pessimism in MBRL.
A standard approach to mitigating the reality gap is quantifying
epistemic uncertainty, often via model ensembles
\citep{chua2018pets}. In Offline RL, where exploration is impossible,
this uncertainty is used to penalize the policy, ensuring it stays
within the support of the data (pessimism) \citep{levine2020offline,
  yu2020mopo, kidambi2020morel}. While effective for avoiding known
bad regions, these methods are inherently passive. Our approach is
\emph{active}: it uses the adversary to explicitly find
high-uncertainty, exploitable regions and then demands more data from
them (via our Error-MDP formulation), rather than just avoiding them.

Robust and Adversarial RL.
Robust control formulations \citep{nilim2005robust,
  wiesemann2013robust} optimize policies against a worst-case
environment chosen from a pre-defined uncertainty set. Our work can be
viewed as learning this uncertainty set simultaneously with the
policy. Adversarial Reinforcement Learning (ARL) typically introduces
an adversary that perturbs dynamics or observations during training
\citep{pinto2017robust}. Our framework differs crucially in its
target: our adversary does not perturb the \emph{state} to break the
policy; it perturbs the \emph{policy} to break the \emph{model}.

Dynamics Adaptation and Sim-to-Real.
Recent approaches have sought to bridge the gap between source and
target dynamics by learning invariant feature representations or
adapting policies directly. For instance, \citet{guo2024off} propose
DARAIL, which leverages domain adaptation and reward-augmented
imitation learning to align the agent's state distribution with that
of an expert in the target domain.  However, such methods typically
assume the existence of optimal expert demonstrations in the target
environment to guide the adaptation.  In contrast, our work addresses
the unsupervised setting of \emph{Simultaneous Policy Learning and
  Sampling} (SPLS), where no target demonstrations are
available. Instead of imitating an expert, our \emph{Strategic SPLS}
agent must actively explore the target environment to identify and
correct discrepancies in the simulator's physics, often targeting
regions a stable expert would strictly avoid.

Value-Aware Model Learning.
Prior work has recognized that not all model errors are
equal. Value-Aware Model Learning (VAML) \citep{farahmand2017value}
proposes weighting the standard predictive loss by the gradient of the
value function, focusing model capacity on states that matter for the
\emph{current} policy. Our work can be viewed as the robust
generalization of VAML. Instead of weighting errors based on a single,
fixed policy, our minimax framework actively finds the
\emph{worst-case} policy that maximizes these errors, ensuring
robustness against future policy shifts. Recent extensions, such as
$\lambda$-models \citep{voelcker2023lambda}, have similarly recognized
the need for iterative updates to resolve the circular dependency
between the learned model and the policy's data distribution.

Reward Learning and Alignment.
The problem of "simulator exploitation" in MBRL is formally isomorphic
to "reward hacking" in methods like RLHF
\citep{christiano2017deep}. In both cases, an agent optimizes a
learned proxy (world model or reward model) and finds adversarial
exploits. Recent theoretical work has begun to unify these
perspectives \citep{GuptaFischDannAgarwal2025}. Our Error-MDP duality
provides a concrete mechanism, active sampling based on critic
disagreement, that is applicable to both domains.

Distributionally Robust Reinforcement Learning. In short, distributionally
robust RL \citep[e.g.,][]{shi2024distributionally} aims to find a robust policy
given a worst-case simulator in a predefined uncertainty set. Our approach
solves the inverse problem: finding a robust simulator that minimizes value
prediction error against a worst-case policy. Even without active data
collection, the resulting model and policy are distinct in general.

There has also been a series of theoretical works which design model-based
algorithms with strong regret or PAC learning guarantees
\citep[e.g.,][]{sun2019model,ayoub2020model,foster2021statistical,
  agarwal2022model}. While our approach and analysis shares many ideas with this
line of work, our focus is slightly different. These works learn a model during
policy optimization to improve sample complexity of learning the optimal policy,
e.g., in factored MDPs. Instead, our goal is to learn a simulator of the
environment, not necessarily tied to optimizing a specific policy but allows us
to accurately estimate the value of any policy in the considered set (see the
main objective in Equation~\eqref{eqn:main_objective})

\section{Analysis of Game with Likelihood Losses}
\label{sec:analysis_game_ll}

\textbf{Assumptions on kernels and losses.}
We assume the state and action spaces are either finite (discrete) or
Borel spaces admitting densities with respect to a fixed dominating
measure. All kernels $\h \sfP(\cdot \mid s, a)$ under consideration are
assumed to be absolutely continuous with respect to the true kernel
$\sfP(\cdot \mid s, a)$ (or more generally to have almost-everywhere
positive densities), so that
$\KL(\sfP(\cdot \mid s, a) \parallel \h \sfP(\cdot \mid s, a))$ is
well-defined (possibly $+\infty$) for every $(s, a)$. When we use the
negative log-likelihood loss
$\cL_\pi(\h \sfP) = \E_{(s, a) \sim d_\pi}[-\log \hat \sfP(s' \mid s, a)]$
we implicitly assume the dominating measure is fixed; in the discrete
case this equals the Shannon entropy decomposition
$\cL_\pi = \E_{d_\pi}[\KL (\sfP \parallel \h \sfP)] + \E_{d_\pi}[H(\sfP)]$
where $H(\sfP) \geq 0$. For continuous spaces where $H(\sfP)$ may be
negative, care must be taken; in such cases, we work with the expected
KL $\E_{d_\pi}[\KL (\sfP \parallel \h \sfP)]$ directly.

\textbf{Assumption on State Space Support.}
We implicitly assume the simulator $\h \sfM$ and true world $\sfM$
share the same state and action spaces $\sS, \sA$. From a theoretical
standpoint, if the simulator places probability mass on states that
are impossible in the real world (or vice versa), divergences like
$\KL$ may become undefined. However, from a strategic standpoint (as
discussed in Section~\ref{sec:sim-metrics}), we only require the
simulator to be accurate on the \emph{reachable} set of states visited
by relevant policies. Our Error-MDP framework
(Section~\ref{sec:active-data-selection}) naturally handles this by
driving exploration only toward regions where errors actually impact
policy value, effectively ignoring parts of the state space that are
irrelevant to the agent.
This issue is particularly relevant in settings like Block MDPs, where
a simulator might operate in a compressed latent space while the real
world operates on high-dimensional observations. In such cases, our
theoretical analysis assumes an oracle mapping exists to define
densities on a common space, though in practice the strategic
robustness objective simply ignores the dimensions that do not affect
value.

\begin{proof}[Proof of Theorem~\ref{th:main}]
  Fix an arbitrary round $t$. By Lemma~\ref{lemma:sim} and
  Lemma~\ref{lemma:pinsker-jensen} (applied to policy $\pi_t$ and
  kernel $\h \sfP_t$) we have
\begin{align*}
\Delta_t = 
\abs*{V_{\pi_t}(\sfP) - V_{\pi_t}(\h \sfP_t)}
& \leq \frac{\gamma R_{\max}}{2(1 - \gamma)} 
\E_{(s, a)\sim d_{\pi_t}}\big[\norm*{\sfP(\cdot\mid s, a)-\h \sfP_t(\cdot\mid s, a)}_1\big] \\
& \leq
  \frac{\gamma R_{\max}}{2(1 - \gamma)^{3/2}}
  \sqrt{2\E_{d_{\pi_t}}\big[\KL(\sfP \parallel \h \sfP_t)\big]} \\
& =
\frac{\gamma R_{\max}}{\sqrt{2}(1 - \gamma)^{3/2}}
\sqrt{ \cL_{\pi_t}(\h \sfP_t) - \E_{d_{\pi_t}}[H(\sfP(\cdot\mid s, a))] }.
\end{align*}
The entropy term $\E_{d_{\pi_t}}[H(\sfP)]$ is non-negative and independent of
the chosen model $\h \sfP_t$, so for the purpose of upper-bounding $\Delta_t$ by
a function of $\cL_{\pi_t}(\h \sfP_t)$ we may drop it. Thus, we can write:
$\Delta_t \leq C'\sqrt{\cL_{\pi_t}(\h \sfP_t)}$, using the shorthand
$C' = \frac{\gamma R_{\max}}{\sqrt{2}(1 - \gamma)^{3/2}}$.

Averaging over $t = 1, \dots, T$ and using Jensen's inequality yields
\[
  \frac{1}{T} \sum_{t = 1}^T \Delta_t
  \leq \frac{1}{T} \sum_{t = 1}^T C' \sqrt{\cL_{\pi_t}(\h \sfP_t)}
  \leq C' \sqrt{ \frac{1}{T}\sum_{t = 1}^T \cL_{\pi_t}(\h \sfP_t) }.
\]
Decomposing the empirical average loss using regret:
\[
\frac{1}{T}\sum_{t = 1}^T \cL_{\pi_t}(\h \sfP_t)
 =
\min_{\h \sfP\in\sM} \curl*{\frac{1}{T} \sum_{t = 1}^T \cL_{\pi_t}(\h \sfP)
+ \frac{\Regret_T}{T}},
\]
and substituting this into the inequality completes the proof. Note
that the term $\min_{\h \sfP \in \sM} \sum \cL_{\pi_t}(\h \sfP)$
represents the approximation error of the model class $\sM$ against
the specific sequence of policies played. It does not assume a single
model exists that is bad for all policies globally; rather, consistent
with standard regret definitions, it compares performance against the
best fixed simulator one could have chosen in hindsight for that
specific sequence.
\end{proof}

\subsection{Auxiliary lemmas}

We first state two lemmas used in the main proof.
The first result is a standard simulation lemma, we give a
self-contained proof to be complete.

\begin{lemma}[Simulation-style bound via TV]
\label{lemma:sim}
Fix a policy $\pi$. For any two transition kernels $\sfP, \h \sfP$ one has
\[
  \abs*{V_\pi(\sfP) - V_\pi(\h \sfP)}
  \leq \frac{\gamma R_{\max}}{1 - \gamma}
  \E_{(s, a)\sim d_\pi}\bracket*{\frac{1}{2}
    \norm*{ \sfP(\cdot\mid s, a) - \h \sfP(\cdot\mid s, a)}_1},
\]
where $V_\pi(\sfP) = \E_{s_0 \sim \rho_0}[V_\pi^\sfP(s_0)]$, $d_\pi$
is the discounted state-action occupancy measure, and
$\norm*{\cdot}_1$ denotes the $\ell_1$ distance between probability
measures.
\end{lemma}

\begin{proof}
  The proof relies on expressing the difference in value functions as
  a discounted sum of single-step prediction errors, which are then
  bounded by the total variation distance.

  Let $V_\pi^\sfP$ and $V_\pi^{\h \sfP}$ denote the value functions
  of a fixed policy $\pi$ under the true dynamics $\sfP$ and the model
  $\h \sfP$, respectively. The value function $V_\pi^\sfP$ satisfies
  the Bellman equation for policy $\pi$:
\[
  V_\pi^\sfP(s)
  = \E_{a \sim \pi(\cdot|s)} \bracket*{ r(s, a)
    + \gamma \E_{s' \sim \sfP(\cdot|s, a)} \bracket*{V_\pi^\sfP(s')} }.
\]
A similar equation holds for $V_\pi^{\h \sfP}$. Let the value
difference be $\delta(s) = V_\pi^\sfP(s) -
V_\pi^{\h \sfP}(s)$. Subtracting the Bellman equation for
$V_\pi^{\h \sfP}$ from that of $V_\pi^\sfP$, we get:
\begin{align*}
  \delta(s)
  & = \E_{a \sim \pi(\cdot|s)} \bracket*{ \gamma \E_{s' \sim \sfP(\cdot|s, a)}\bracket*{V_\pi^\sfP(s')} - \gamma \E_{s' \sim \h \sfP(\cdot|s, a)}\bracket*{V_\pi^{\h \sfP}(s')} } \\
  & = \gamma \, \E_{a \sim \pi(\cdot|s)} \bracket*{ \E_{s' \sim \sfP}\bracket*{V_\pi^\sfP(s')} - \E_{s' \sim \sfP}\bracket*{V_\pi^{\h \sfP}(s')} + \E_{s' \sim \sfP}\bracket*{V_\pi^{\h \sfP}(s')} - \E_{s' \sim \h \sfP}\bracket*{V_\pi^{\h \sfP}(s')} } \\
  & = \gamma \, \E_{a \sim \pi(\cdot|s)} \bracket*{ \E_{s' \sim \sfP(\cdot|s, a)}\bracket*{\delta(s')} + \sum_{s'}\paren*{\sfP(s'|s, a) - \h \sfP(s'|s, a)} V_\pi^{\h \sfP}(s') }.
\end{align*}
Let the single-step error term be
$\e(s, a) = \sum_{s'}(\sfP(s'|s, a) - \h \sfP(s'|s, a))
V_\pi^{\h \sfP}(s')$. The equation for $\delta(s)$ becomes a
recursive relationship:
\[
\delta(s) = \E_{a \sim \pi(\cdot|s)}\bracket*{ \gamma \E_{s' \sim \sfP(\cdot|s, a)}\bracket*{\delta(s')} + \gamma \e(s, a) }.
\]
If we unroll this recursion starting from a state $s_0$, the expected
value difference is the expected discounted sum of the single-step
errors $\e(s_t, a_t)$ over a trajectory $\tau$ generated under the
true dynamics $\sfP$:
\[
  \delta(s_0)
  = \E_{\tau \sim (s_0, \pi, \sfP)} \bracket*{ \sum_{t=0}^\infty \gamma^{t + 1}
    \e(s_t, a_t) }.
\]
Taking an expectation over an initial state distribution $\rho_0(s_0)$
gives the overall expected value difference:
\begin{align*}
\E_{s_0 \sim \rho_0}\bracket*{ V_\pi^\sfP(s_0) - V_\pi^{\h \sfP}(s_0) } &= \E_{\tau \sim (\rho_0, \pi, \sfP)} \bracket*{ \sum_{t=0}^\infty \gamma^{t + 1} \e(s_t, a_t) } \\
& = \gamma \sum_{s, a} \E_{\tau}\bracket*{ \sum_{t=0}^\infty \gamma^t \mathbf{1}\curl*{(s_t,a_t)=(s, a)} } \e(s, a) \\
& = \gamma \sum_{s, a} d_\pi(s, a) \, \e(s, a).
\end{align*}
Now, we take the absolute value and bound the error term $\e(s,
a)$. Since rewards are in $[0, R_{\max}]$, any value function is
bounded in $[0, \frac{R_{\max}}{1 - \gamma}]$. The range of values of
$V_\pi^{\h \sfP}$ is therefore at most $\frac{R_{\max}}{1 -
  \gamma}$. The term $\e(s, a)$ is the difference in the expected
value of the function $V_\pi^{\h \sfP}(\cdot)$ under two different
probability distributions. This difference can be bounded by half the
range of the function multiplied by the $\ell_1$ distance between the
distributions:
\[
  \abs*{\e(s, a)} = \abs*{ \sum_{s'}\paren*{\sfP(s'|s, a)
      - \h \sfP(s'|s, a)} V_\pi^{\h \sfP}(s') }
  \leq \frac{1}{2} \frac{R_{\max}}{1 - \gamma}
  \norm*{\sfP(\cdot|s, a) - \h \sfP(\cdot|s, a)}_1.
\]
Substituting this back into our expression for the value difference gives:
\begin{align*}
  \abs*{ V_\pi(\sfP) - V_\pi(\h \sfP) }
  & = \abs*{ \E_{s_0 \sim \rho_0}\bracket*{ V_\pi^\sfP(s_0) - V_\pi^{\h \sfP}(s_0) } } \\
  & \leq \gamma \sum_{s, a} d_\pi(s, a) \abs*{\e(s, a)} \\
  & \leq \gamma \sum_{s, a} d_\pi(s, a) \paren*{ \frac{1}{2} \frac{R_{\max}}{1 - \gamma} \norm*{\sfP(\cdot|s, a) - \h \sfP(\cdot|s, a)}_1 } \\
  & = \frac{\gamma R_{\max}}{1 - \gamma} \E_{(s, a)\sim d_\pi} \bracket*{ \frac{1}{2} \norm*{\sfP(\cdot|s, a) - \h \sfP(\cdot|s, a)}_1 }.
\end{align*}
This completes the proof.
\end{proof}
While Lemma~\ref{lemma:sim} is a standard result, recent work has
demonstrated that the dependence on the horizon and error terms in
such simulation bounds is structurally tight
\citep{lobel2024simulation}, suggesting that improvements must come
from the learning objective, such as our minimax formulation, rather
than tighter analysis of the simulation gap itself.

\begin{lemma}
\label{lemma:pinsker-jensen}
For any policy $\pi$ and kernel $\h \sfP$,
\[
  \E_{(s, a) \sim d_\pi}\big[\norm*{\sfP(\cdot \mid s, a)
    - \h \sfP(\cdot \mid s, a)}_1\big]
\leq
\sqrt{\frac{2}{1 - \gamma}\E_{(s, a) \sim d_\pi}\big[\KL(\sfP(\cdot \mid s, a)
  \parallel \h \sfP(\cdot \mid s, a))\big]}.
\]
\end{lemma}

\begin{proof}
  By Sch\"utzenberger-Pinsker's inequality \citep{Rioul2023}, for each
  $(s, a)$, we have
\[
  \norm*{\sfP(\cdot\mid s, a) - \h \sfP(\cdot\mid s, a)}_1
  \leq \sqrt{2 \KL \paren*{\sfP(\cdot\mid s, a)
      \parallel \h \sfP(\cdot \mid s, a)}}.
\]
Taking expectation with respect to $(s, a) \sim d_\pi$ and using
Jensen's inequality and the concavity of $\sqrt{\cdot}$ yields
\[
  \E_{d_\pi}\bracket*{\norm*{\sfP - \h \sfP}_1}
  \leq \E_{d_\pi}\big[ \sqrt{2\KL} \big] \leq \sqrt{\frac{2}{1 - \gamma} \E_{d_\pi}\bracket*{\KL}}.
\]
This completes the proof.
\end{proof}

\subsection{Lower bounds}
\label{sec:lower-bounds}

The rate of convergence derived in Theorem~\ref{th:main}, which scales
as $O(T^{-1/4})$, raises a natural question: is this rate fundamental
to the problem of learning simulators, or is it an artifact of the
analysis? We argue that this rate represents a fundamental limitation
of the standard \emph{Maximum Likelihood Estimation (MLE)} paradigm,
stemming from the geometric mismatch between the learning objective
and the control objective.

Standard model-based RL separates learning into two phases:
\begin{enumerate}

\item \textbf{Estimation:} The model $\h \sfP$ minimizes the KL
  divergence (via MLE) to the true dynamics $\sfP$:
  \[
    \min_{\h \sfP} \E_{d_\pi} [\KL(\sfP \parallel \h \sfP)].
  \]

\item \textbf{Control:} The agent maximizes a value function, whose
  error scales with the Total Variation (TV) distance
  (Lemma~\ref{lemma:sim}):
  \[
    |V_\pi(\sfP) - V_\pi(\h \sfP)|
    \leq \frac{\gamma R_{\max}}{1-\gamma} \E_{d_\pi} [\TV(\sfP, \h \sfP)].
  \]

\end{enumerate}To translate the guarantee from phase 1 (KL) to phase 2 (TV), one must
invoke Sch\"utzenberger-Pinsker's inequality \citep{Rioul2023}:
$\TV(P, Q) \leq \sqrt{\frac{1}{2}\KL(P\|Q)}$.  This inequality
introduces a square root that degrades the convergence rate.  This
square-root dependence is optimal near zero \citep{Rioul2023}.  If a
learner achieves a standard fast rate of $\e \approx O(T^{-1/2})$ for
the KL-divergence (e.g., using parametric MLE), the resulting
guarantee for the value gap degrades to:
\[
\text{Value Gap} \approx \sqrt{O(T^{-1/2})} = O(T^{-1/4})
\]
This \emph{Sch\"utzenberger-Pinsker tax} implies that MLE is
theoretically inefficient for control, requiring quadratically more
samples to achieve the same guarantee on policy performance compared
to a method that optimizes the relevant metric directly.

This tax is not merely hypothetical; it reflects a divergence in
convergence rates for online learning. In the general agnostic
setting, standard no-regret algorithms guarantee that the average loss
decays at a rate of $O(T^{-1/2})$.
\begin{itemize}

\item \textbf{MLE Approach:} Minimizes KL. Achieving an average KL
  error of $O(T^{-1/2})$ implies, via Sch\"utzenberger-Pinsker's
  inequality, a TV error of only $\sqrt{O(T^{-1/2})} = O(T^{-1/4})$.

\item \textbf{Minimax Approach:} Minimizes TV directly. By using a
  loss function that scales linearly with the metric of interest
  (e.g., Wasserstein or dual-TV), one can theoretically recover the
  standard $O(T^{-1/2})$ rate for the control-relevant error.
\end{itemize}
This $O(T^{-1/4})$ vs. $O(T^{-1/2})$ gap highlights that MLE-based
simulators are inherently sample-inefficient for control tasks
compared to value-aware learners. We formalize this as a strict
mathematical separation between MLE and metric-aware learners below.

Our proposed Minimax formulation
(Section~\ref{sec:game-simplification}) circumvents this tax by
defining the game directly over the transport metric (TV or
Wasserstein) that bounds the value difference.
\[
  \min_{\h \sfP} \max_{D} \E_{(s, a, s') \sim d}
  \bracket*{ D(s, a, s') - \E_{\h s' \sim \h \sfP}[D(s, a, \h s')] }
\]
Since the loss function \emph{is} the metric of interest, no
conversion inequality is required. Standard results for convex-concave
games suggest that such objectives can theoretically achieve the
faster $O(T^{-1/2})$ rates, verifying that our approach is not just
strategically robust, but potentially sample-efficient in a way that
MLE cannot be.

We argued above that minimizing KL divergence (via MLE) incurs a
\emph{Sch\"utzenberger-Pinsker tax} that degrades the convergence rate for
control tasks from $O(T^{-1/2})$ to $O(T^{-1/4})$. Here, we provide a rigorous
justification for this claim by exhibiting a strict mathematical separation
between MLE-based and metric-aware learners.

\begin{proposition}[Lower Bound for Metric-Mismatched Regret]
\label{prop:lower_bound}
Consider the online learning of a sequence of probability distributions
$P_1, \dots, P_T$ selected by an adversary from a convex class $\cM$. There
exists a randomized i.i.d.\ adversarial sequence such that:
\begin{enumerate}
\item An MLE-based learner (minimizing KL divergence) achieves the
  mathematically optimal expected cumulative KL regret of $\cO(\log T)$, yet
  inherently suffers an expected average Total Variation regret of
  $\Omega(T^{-1/4})$.

\item In contrast, an algorithm directly minimizing the Total Variation metric
  achieves an expected average Total Variation regret bounded by
  $\cO(T^{-1/2})$.
\end{enumerate}
This demonstrates a strict mathematical separation: bounding KL regret
is fundamentally insufficient for optimal control, as the geometric
mismatch forces MLE to incur the Sch\"utzenberger-Pinsker tax even
against a simple i.i.d.\ sequence.
\end{proposition}

\begin{proof}
The structural gap arises because KL divergence locally penalizes
errors similarly to squared loss ($L_2$), tracking the \emph{mean} of
a sequence, whereas Total Variation acts as absolute loss ($L_1$),
tracking the \emph{median}.

Let $\cM$ contain the family of Bernoulli distributions
$P_x = \mathrm{Ber}(x)$ for $x \in [1/4, 3/4]$. In this domain, the
Total Variation distance is simply absolute error:
$\TV(P_x, P_y) = |x - y|$.

Let $\delta = T^{-1/4}$ (assuming $T$ is large enough so
$\delta < 1/4$). The adversary plays an i.i.d.\ sequence:
$x_t = 3/4$ with probability $\delta$, and $x_t = 1/4$ with
probability $1-\delta$.

An MLE learner minimizes the expected KL loss. Because the KL
divergence is strongly convex for parameters restricted to
$[1/4, 3/4]$, standard online algorithms like Follow-The-Leader (FTL)
achieve an expected cumulative KL regret of $\cO(\log T)$.

We now evaluate the expected TV error of this MLE prediction. For
$t \ge 2$, unregularized FTL predicts the exact empirical mean
$\hat{x}_t = \frac{1}{t-1}\sum_{i=1}^{t-1} x_i$. Because
$x_i \in \{1/4, 3/4\}$, the empirical mean is always within
$[1/4, 3/4]$. The expected TV error of any prediction
$x \in [1/4, 3/4]$ against the random draw $x_t$ is:
\begin{align*}
  f(x)
  & = \E_{x_t}[\TV(P_{x_t}, P_x)] 
    = \delta \left( \frac{3}{4} - x \right)
    + (1-\delta) \left( x - \frac{1}{4} \right)
    = x(1-2\delta) + \delta - \frac{1}{4}.
\end{align*}
Since $f(x)$ is exactly an affine function on $[1/4, 3/4]$, taking the full
expectation over the learner's history simply evaluates this affine function at
$\E[\hat{x}_t]$. Since FTL is an unbiased estimator of the mean,
$\E[\hat{x}_t] = \mu = \delta(3/4) + (1-\delta)(1/4) = 1/4 + \delta/2$.  The
exact expected TV error of MLE at any step $t \ge 2$ is therefore evaluated
precisely at $\mu$:
\[
\E[f(\hat{x}_t)] = f(\E[\hat{x}_t])
= f\left(\frac{1}{4} + \frac{\delta}{2}\right)
= \left(\frac{1}{4} + \frac{\delta}{2}\right)(1-2\delta)
+ \delta - \frac{1}{4}
= \delta - \delta^2.
\]
Conversely, the optimal fixed model in hindsight $x^*$ for the TV
metric minimizes expected absolute error. Since
$f'(x) = 1-2\delta > 0$, the minimum on $[1/4, 3/4]$ is attained at
the boundary $x^* = 1/4$, which is the median (since
$1-\delta > 1/2$). The expected TV error of this optimal model is:
\[
f(x^*) = f(1/4) = \frac{\delta}{2}.
\]
Therefore, the expected average TV regret of the MLE algorithm over
$T$ steps is the exact difference (up to a negligible
$\cO(1/T)$ boundary term at $t=1$):
\[
\E[\text{Avg TV Regret}_{\text{MLE}}]
= \frac{1}{T} \sum_{t=1}^T
\bigl( \E[f(\hat{x}_t)] - f(x^*) \bigr)
= (\delta - \delta^2) - \frac{\delta}{2}
\pm \cO(1/T)
= \frac{\delta}{2} - \delta^2
\pm \cO(1/T)
= \Omega(T^{-1/4}).
\]
Meanwhile, an algorithm running standard no-regret online learning
directly on the TV metric (e.g., Online Gradient Descent on absolute
loss) will correctly track the median $x^*=1/4$. Since the per-round
loss $\ell_t(x) = |x_t - x|$ is convex and $1$-Lipschitz on
$[1/4, 3/4]$, the standard online convex optimization guarantee
yields a total TV regret of $\cO(\sqrt{T})$, hence an expected
average TV regret of $\cO(T^{-1/2})$. This establishes a strict
mathematical separation, proving that MLE fundamentally incurs a
severe TV tax.
\end{proof}

It is important to distinguish this result from standard batch
(I.I.D.) estimation of a \emph{single fixed} distribution.
\begin{itemize}

\item \textbf{I.I.D.\ Estimation Setting:} In standard statistical
  estimation, where all data comes from a single unknown distribution
  $P$ and the goal is to estimate $P$ itself, MLE is asymptotically
  efficient and achieves ``fast rates'' of $O(1/T)$ for the KL
  divergence. In this specific case, the TV error decays as
  $\sqrt{O(1/T)} = O(T^{-1/2})$. Since direct TV minimization also
  typically achieves $O(T^{-1/2})$, there is no rate penalty in the
  batch estimation setting.

\item \textbf{Online Prediction Setting:} Our work focuses on the
  online prediction setting
  (Section~\ref{sec:online-learning-guarantees}), where the
  distribution sequence is chosen adaptively by a Policy Player.
  Critically, Proposition~\ref{prop:lower_bound} shows that the
  Sch\"utzenberger-Pinsker tax arises even against a simple
  i.i.d.\ adversary: MLE incurs $\Omega(T^{-1/4})$ average TV regret
  while a direct TV minimizer achieves $O(T^{-1/2})$.

\end{itemize}
Thus, in the online setting relevant to robust control, the geometric mismatch
between KL and TV incurs a genuine rate penalty ($T^{-1/4}$ vs $T^{-1/2}$),
justifying the need for the direct Minimax formulation proposed in this paper.

A theoretical counter-argument might suggest that because the KL
divergence is strongly convex with respect to the $L_1$ norm (or
exp-concave), one could achieve logarithmic regret $O(\log T)$,
corresponding to an average KL rate of $O(\log T / T)$. If one then
applied the Pinsker inequality $\TV \leq \sqrt{\KL/2}$, the TV error
per round would scale as $\sqrt{\log T/T}$, apparently recovering a
near-$T^{-1/2}$ rate.
However, Proposition~\ref{prop:lower_bound} directly refutes this
argument: even though MLE achieves the \emph{optimal}
$\cO(\log T)$ KL regret in our construction, its TV regret
remains $\Omega(T^{-1/4})$. The reason is that low KL regret ensures
MLE tracks the \emph{mean} of the distribution sequence (analogous to
$L_2$ optimality), but the TV-relevant quantity is the \emph{median}
(analogous to $L_1$ optimality). These two can diverge significantly,
and the Pinsker inequality cannot close the gap because it bounds
\emph{absolute} TV error, not TV \emph{regret}.

Moreover, in the context of learning deep simulators, even the
$\cO(\log T)$ KL rate is generally inaccessible for two
practical reasons:
\begin{enumerate}

\item \textbf{Non-Convex Parameterization:} While the negative
  log-likelihood is convex in the space of probability distributions
  $P$, it is highly non-convex with respect to the parameters $\theta$
  of a neural network. Standard gradient-based algorithms (like Adam
  or SGD) used in Deep RL do not satisfy the strong convexity
  assumptions required for logarithmic regret in parameter space.

\item \textbf{High Dimensionality:} Algorithms that theoretically
  achieve fast rates for exp-concave losses (such as the Online Newton
  Step) typically incur computational costs or regret bounds that
  scale poorly with the dimension $d$ of the parameter space (e.g.,
  $O(d \log T)$ regret with $O(d^2)$ or $O(d^3)$ runtime). For modern
  world models where $d$ is in the millions, these methods are
  computationally intractable.
\end{enumerate}
Consequently, practical Deep Model-Based Reinforcement Learning faces
a double penalty: the fundamental metric mismatch demonstrated in
Proposition~\ref{prop:lower_bound}, compounded by the practical
inability to achieve fast KL rates. Direct Minimax optimization of the
transport metric sidesteps both issues.

\section{Additional Discussion and Analysis of Games using Upper-Bounds }
\subsection{Limitations}
\label{sec:limitations}

Our approach assumes that active sampling in the real environment is possible
during simulator learning. This is true in many real-world applications but
there may be cases where the sampling distribution is more constraint or even
fixed. Our approach could still be used in such cases, e.g., by dataset
reweighting but we did not investigate this.

While our algorithmic approach is broadly applicable to various settings and
tasks, our empirical evaluation only covers control tasks as they are meaningful
benchmarks with sufficient complexity and allow assessing the methodological
improvements. However, a more extensive evaluation on other domains such as
generative language models could be of interest in future work.

\subsection{Proof of Lemma~\ref{lemma:general-coverage-bound}}
\begin{proof}
  The proof is a direct application of the $\kappa$-coverage assumption. We use
  $\E_{d_\pi}[\cdot]$ to denote the sum weighted by the unnormalized occupancy
  $d_\pi$:
  \begin{align*}
    \abs*{V_\pi(\sfP) - V_\pi(\h\sfP)}
    & \leq \E_{(s, a)\sim d_\pi} \bracket*{ C(\pi, \h \sfP, s, a) \cdot \wt \sfD(\sfP, \h\sfP) } \\
    & = \sum_{s, a} d_\pi(s, a) \cdot C(\pi, \h \sfP, s, a) \cdot \wt \sfD(\sfP, \h\sfP) \\
    & \leq \sum_{s, a} \kappa d_{\mathrm{data}}(s, a) \cdot C(\pi, \h \sfP, s, a) \cdot \wt \sfD(\sfP, \h\sfP) \tag{by $\kappa$-coverage} \\
    & = \kappa \E_{(s, a)\sim d_{\mathrm{data}}} \bracket*{ C(\pi, \h \sfP, s, a) \cdot \wt \sfD(\sfP, \h\sfP) }.
  \end{align*}
  Taking the supremum over $\pi \in \Pi$ on both sides completes the proof.
\end{proof}

\subsection{TV Coverage Guarantee}
\begin{corollary}[TV Coverage Guarantee]
  \label{cor:coverage-TV}
  Assume rewards $r(s, a)\in[0,R_{\max}]$ and discount $\gamma\in(0,1)$.  Let
  $\Pi$ be a policy class of interest, and assume $d_{\mathrm{data}}$ satisfies
  $\kappa$-coverage for $\Pi$.  Then
  \[
    \sup_{\pi\in\Pi}\abs*{V_\pi(\sfP) - V_\pi(\h\sfP)} \le \frac{2\gamma
      R_{\max}}{1 - \gamma}\; \kappa \cdot \E_{(s, a)\sim d_{\mathrm{data}}}
    \bracket*{ \TV\paren*{\sfP(\cdot|s, a),\h\sfP(\cdot|s, a)} }.
  \]
\end{corollary}
\begin{proof}
  This is an instantiation of Lemma~\ref{lemma:general-coverage-bound}.  From
  Lemma~\ref{lemma:sim}, we have $\wt \sfD = \frac{1}{2}\norm{\cdot}_1 = \TV$
  and the cost term
  $C(\pi, \h \sfP, s, a) = \frac{\gamma R_{\max}}{1 - \gamma}$. This cost $C$ is
  a uniform constant, independent of $\pi$, $s$, and $a$.  Plugging this $C$ and
  $\wt \sfD$ into the result of Lemma~\ref{lemma:general-coverage-bound} gives:
  \begin{align*}
    \sup_{\pi\in\Pi}\abs*{V_\pi(\sfP) - V_\pi(\h\sfP)}
    & \leq \kappa \cdot \sup_{\pi\in\Pi} \E_{(s, a)\sim d_{\mathrm{data}}} \bracket*{ \frac{\gamma R_{\max}}{1 - \gamma} \cdot \TV(\sfP, \h\sfP) } \\
    & = \frac{\gamma R_{\max}}{1 - \gamma} \kappa \E_{(s, a)\sim d_{\mathrm{data}}} \bracket*{ \TV(\sfP, \h\sfP) }.
  \end{align*}
  This completes the proof.
\end{proof}

\subsection{Proof of Theorem~\ref{th:minimax_tv}}
\begin{proof}
  All three parts are straightforward consequences of the coverage bound
  (Corollary~\ref{cor:coverage-TV}) combined with elementary algebra.

  Part (A). Applying Corollary~\ref{cor:coverage-TV} and the definition of
  $\e_{\mathrm{critic}}$:
  \begin{align*}
    \sup_{\pi\in\Pi}\abs{V_\pi(\sfP) - V_\pi(\h\sfP^*)}
    & \leq \frac{\gamma R_{\max}}{1 - \gamma} \kappa \E_{d_{\mathrm{data}}}[\TV(\sfP, \h\sfP^*)] \\
    & = \frac{\gamma R_{\max}}{1 - \gamma} \kappa \E_{d_{\mathrm{data}}}[\tfrac{1}{2} \sup_{\|D\|_\infty \leq 1} \Delta_D(s, a)] \\
    & \leq \frac{\gamma R_{\max}}{1 - \gamma} \kappa \tfrac{1}{2} \paren*{ \E_{d_{\mathrm{data}}}[\sup_{D \in \cD} \Delta_D(s, a)] + \e_{\mathrm{critic}} } \\
    & = \frac{\gamma R_{\max}}{1 - \gamma} \kappa \paren*{ \tfrac{1}{2} J_\cD(\h\sfP^*) + \tfrac{1}{2} \e_{\mathrm{critic}} } \\
    & \leq \frac{\gamma R_{\max}}{1 - \gamma} \kappa \paren*{ \tfrac{1}{2} \e_{\mathrm{model}} + \tfrac{1}{2} \e_{\mathrm{critic}} }.
  \end{align*}

  Part (B). With probability at least $1 - \delta$, we have
  \[
    J_{\cD}(\h \sfP) \leq \h J_N(\h\sfP) + \omega \leq \inf_{\h \sfP \in \sM}\h
    J_N(\h \sfP) + \eta + \omega \leq \inf_{\h \sfP \in \sM}J_{\cD}(\h \sfP) +
    2\omega + \eta.
  \]
  Plugging this into the result from Part (A) as $\e_{\mathrm{model}}$ gives the
  finite-sample bound.
\end{proof}

\subsection{Proof of Theorem~\ref{th:minimax_wasserstein}}
\begin{proof}
  The proof mirrors Theorem~\ref{th:minimax_tv}. We apply
  Corollary~\ref{cor:coverage-wasserstein} and the definition of
  $\e_{\mathrm{critic}}$:
  \begin{align*}
    \sup_{\pi\in\Pi}\abs{V_\pi(\sfP) - V_\pi(\h\sfP^*)}
    & \leq \gamma L_v \kappa \E_{d_{\mathrm{data}}}[W_1(\sfP, \h\sfP^*)] \\
    & = \gamma L_v \kappa \E_{d_{\mathrm{data}}}[\sup_{\|D\|_{Lip} \leq 1} \Delta_D(s, a)] \\
    & \leq \gamma L_v \kappa \paren*{ \E_{d_{\mathrm{data}}}[\sup_{D \in \cD_{\text{Lip}}} \Delta_D(s, a)] + \e_{\mathrm{critic}} } \\
    & = \gamma L_v \kappa \paren*{ J_{\cD_{\text{Lip}}}(\h\sfP^*) + \e_{\mathrm{critic}} } \\
    & \leq \gamma L_v \kappa \paren*{ \e_{\mathrm{model}} + \e_{\mathrm{critic}} }.
  \end{align*}
  This completes the proof.
\end{proof}

\section{Active data collection}
\subsection{Proof of Theorem~\ref{th:error-mdp-duality}}
\begin{proof}
  The proof follows directly from the bound derived in
  Corollary~\ref{cor:coverage-wasserstein}, before the application of the
  coverage constant $\kappa$. From that proof, we have for any policy $\pi$:
  \[
    \abs*{V_\pi(\sfP) - V_\pi(\h \sfP)} \leq \gamma L_v \E_{(s, a) \sim d_\pi}
    \bracket*{ W_1(\sfP(\cdot \mid s, a), \h\sfP(\cdot \mid s, a)) }.
  \]
  We recognize the expectation term as the definition of the value of $\pi$ in
  the Error-MDP:
  \[
    \E_{(s, a) \sim d_\pi} \bracket*{W_1(\sfP(\cdot \mid s, a), \h\sfP(\cdot\mid
      s, a)) } = \E_{(s, a) \sim d_\pi} \bracket*{ r_{\text{err}}(s, a) } =
    V_{\pi}(\sfM_{\text{err}}).
  \]
  Substituting this back, we get:
  \[
    \abs*{V_\pi(\sfP) - V_\pi(\h \sfP)} \leq \gamma L_v
    V_{\pi}(\sfM_{\text{err}}).
  \]
  Since this holds for all $\pi \in \Pi$, we can take the supremum over all
  policies on both sides. The supremum of the left-hand side is our
  objective. The supremum of the right-hand side is, by definition, the optimal
  value of the Error-MDP:
  \[
    \sup_{\pi \in \Pi} \abs*{V_\pi(\sfP) - V_\pi(\h \sfP)} \leq \sup_{\pi \in
      \Pi} \paren*{ \gamma L_v V_{\pi}(\sfM_{\text{err}}) } = \gamma L_v
    \sup_{\pi \in \Pi} V_{\pi}(\sfM_{\text{err}}) = \gamma L_v
    V^*(\sfM_{\text{err}}).
  \]
  This completes the proof.
\end{proof}

\subsection{Proof of Theorem~\ref{th:finite_time_active}}
\begin{proof}
  The proof combines the convexity of the Wasserstein distance $W_1$
  in its second argument with the Wasserstein simulation bound and the
  uniform coverage assumption.

  By the Wasserstein analogue of the simulation lemma
  (Lemma~\ref{lemma:sim}), for any fixed policy $\pi \in \Pi$ and the
  averaged model
  $\ov{\h \sfP}_T = \frac{1}{T}\sum_{t=1}^T \h \sfP_t$:
  \[
    \abs*{V_\pi(\sfP) - V_\pi(\ov{\h \sfP}_T)}
    \leq \gamma L_v
    \E_{(s,a) \sim d_\pi}\Big[
      W_1\big(\sfP(\cdot|s,a),\, \ov{\h \sfP}_T(\cdot|s,a)\big)
    \Big].
  \]
  Since $W_1$ is convex in its second argument (as a supremum of
  linear functionals via the Kantorovich--Rubinstein duality), for each
  $(s, a)$:
  \[
    W_1\big(\sfP(\cdot|s,a),\, \ov{\h \sfP}_T(\cdot|s,a)\big)
    \leq \frac{1}{T}\sum_{t=1}^T
    W_1\big(\sfP(\cdot|s,a),\, \h \sfP_t(\cdot|s,a)\big).
  \]
  Substituting and applying the uniform coverage assumption
  ($d_\pi \leq \kappa\, d_t$ for all $\pi \in \Pi$ and $t$):
  \begin{align*}
    \sup_{\pi \in \Pi} \abs*{V_\pi(\sfP) - V_\pi(\ov{\h \sfP}_T)}
    &\leq \gamma L_v \sup_{\pi}
      \frac{1}{T}\sum_{t=1}^T
      \E_{d_\pi}\Big[ W_1\big(\sfP(\cdot|s,a),\, \h \sfP_t(\cdot|s,a)\big) \Big] \\
    &\leq \gamma L_v \kappa\,
      \frac{1}{T}\sum_{t=1}^T
      \E_{(s, a) \sim d_t}\Big[ W_1\big(\sfP(\cdot|s,a),\,
      \h \sfP_t(\cdot|s,a)\big) \Big].
  \end{align*}
  By the duality of Wasserstein distance and the definition of the
  critic approximation error $\e_{\mathrm{critic}}$:
  \[
    \E_{(s, a)\sim d_t}\Big[ W_1(\cdot, \cdot) \Big]
    \leq \E_{(s, a)\sim d_t}
    \Big[ \sup_{D\in\cD_{L_v}} \Delta_D(s, a) \Big]
    + \e_{\mathrm{critic}},
  \]
  where
  $\Delta_D(s, a) = \E_{s' \sim \sfP}[D(s, a, s')] - \E_{s'\sim \h
    \sfP_t}[D(s, a, s')]$.  Substituting back:
  \begin{align*}
    \sup_{\pi \in \Pi} \abs*{V_\pi(\sfP) - V_\pi(\ov{\h \sfP}_T)}
    &\leq \gamma L_v \kappa\,
      \frac{1}{T}\sum_{t=1}^T
      \paren*{\E_{d_t}\Big[\sup_{D\in\cD_{L_v}} \Delta_D\Big]
      + \e_{\mathrm{critic}}} \\
    &= \gamma L_v \kappa
      \paren*{\frac{1}{T}\sum_{t=1}^T
      \E_{d_t}\Big[\sup_{D\in\cD_{L_v}} \Delta_D\Big]
      + \e_{\mathrm{critic}}} \\
    &= \gamma L_v \kappa
      \paren*{\ov{J}_{\cD_{L_v}} + \e_{\mathrm{critic}}}.
  \end{align*}
  This completes the proof.
\end{proof}

\subsection{Proof of Theorem~\ref{th:convergence}}
\begin{proof}
  Let $L_t(\h\sfP, d) = L(\h\sfP, d)$ be the stationary payoff function.  The
  Model-Player (minimizer) incurs loss sequence $\ell_t(\cdot) = L(\cdot,
  d_t)$. By the definition of regret against the best fixed model in $\sM$:
  \begin{equation} \label{eq:regret_m} \sum_{t=1}^T L(\h\sfP_t, d_t) -
    \min_{\h\sfP^* \in \sM} \sum_{t=1}^T L(\h\sfP^*, d_t) \leq \Regret_M(T).
  \end{equation}
  The Distribution-Player (maximizer) receives payoff sequence
  $g_t(\cdot) = L(\h\sfP_t, \cdot)$. By the definition of regret against the
  best fixed distribution in $\Delta$:
  \begin{equation} \label{eq:regret_d} \max_{d^* \in \Delta} \sum_{t=1}^T
    L(\h\sfP_t, d^*) - \sum_{t=1}^T L(\h\sfP_t, d_t) \leq \Regret_D(T).
  \end{equation}
  Summing \eqref{eq:regret_m} and \eqref{eq:regret_d} cancels the interaction
  term $\sum_{t=1}^T L(\h\sfP_t, d_t)$:
  \begin{equation} \label{eq:regret_sum} \max_{d^* \in \Delta} \sum_{t=1}^T
    L(\h\sfP_t, d^*) - \min_{\h\sfP^* \in \sM} \sum_{t=1}^T L(\h\sfP^*, d_t)
    \leq \Regret_M(T) + \Regret_D(T).
  \end{equation}
  We now rely on the convexity-concavity of $L(\h\sfP, d)$:
  \begin{itemize}

  \item $L(\cdot, d)$ is convex with respect to $\h\sfP$ (as it is an
    expectation of the convex $W_1$ distance). By Jensen's inequality:
    \[
      \sum_{t=1}^T L(\h\sfP_t, d^*) \geq T \cdot L
      \paren*{\frac{1}{T}\sum_{t=1}^T \h\sfP_t, d^*} = T \cdot L(\ov{\h\sfP}_T,
      d^*).
    \]
  \item $L(\h\sfP, \cdot)$ is concave with respect to $d$ (as it is a linear
    expectation minus a convex KL divergence). By Jensen's inequality:
    \[
      \sum_{t=1}^T L(\h\sfP^*, d_t) \leq T \cdot L\left(\h\sfP^*,
        \frac{1}{T}\sum_{t=1}^T d_t\right) = T \cdot L(\h\sfP^*, \ov{d}_T).
    \]
  \end{itemize}
  Substituting these bounds into \eqref{eq:regret_sum} and dividing by $T$
  yields:
  \[
    \max_{d^* \in \Delta} L(\ov{\h\sfP}_T, d^*) - \min_{\h\sfP^* \in \sM}
    L(\h\sfP^*, \ov{d}_T) \leq \frac{\Regret_M(T) + \Regret_D(T)}{T}.
  \]
  This completes the proof.
\end{proof}

\section{Lipschitz Continuity of Value Functions}
\label{app:lipschitz}

In Section~\ref{sec:game-simplification}
(Corollary~\ref{cor:coverage-wasserstein} and
Theorem~\ref{th:minimax_wasserstein}), we assumed that the value function
$V_\pi^{\h \sfP}$ is $L_v$-Lipschitz with respect to the state metric $d$. We
now show that this property follows from standard smoothness assumptions on the
reward function, transition dynamics, and policy.

\begin{proposition}[Lipschitz Constant of Value Function]
  \label{prop:value_lipschitz}
  Let $(\sS, d_\sS)$ and $(\sA, d_\sA)$ be metric spaces. Assume the
  reward function $r(s,a)$ is $L_r^s$-Lipschitz in $s$ uniformly over
  $a$ and $L_r^a$-Lipschitz in $a$ uniformly over $s$:
  \[
    \abs{r(s,a) - r(s', a)} \leq L_r^s \, d_\sS(s, s'), \quad
    \abs{r(s,a) - r(s, a')} \leq L_r^a \, d_\sA(a, a'), \quad
    \forall s, s' \in \sS,\; a, a' \in \sA.
  \]
  Assume the transition kernel $\h \sfP(\cdot|s,a)$ is
  $L_P^s$-Lipschitz in $s$ and $L_P^a$-Lipschitz in $a$ with respect
  to the 1-Wasserstein distance $W_1$:
  \[
    W_1(\h \sfP(\cdot|s,a), \h \sfP(\cdot|s', a)) \leq L_P^s \, d_\sS(s, s'),
    \quad
    W_1(\h \sfP(\cdot|s,a), \h \sfP(\cdot|s, a')) \leq L_P^a \, d_\sA(a, a'),
  \]
  for all $s, s' \in \sS$ and $a, a' \in \sA$.
  Assume the policy $\pi(\cdot|s)$ is $L_\pi$-Lipschitz in $s$ with
  respect to $W_1$ on $\sA$:
  \[
    W_1(\pi(\cdot|s), \pi(\cdot|s')) \leq L_\pi \, d_\sS(s,s'),
    \quad \forall s, s' \in \sS.
  \]
  If $\gamma (L_P^s + L_\pi L_P^a) < 1$, then the value function
  $V_\pi^{\h \sfP}$ is $L_v$-Lipschitz with constant:
  \[
    L_v = \frac{L_r^s + L_\pi L_r^a}{1 - \gamma (L_P^s + L_\pi L_P^a)}.
  \]
\end{proposition}

\begin{proof}
  The value function $V_\pi^{\h \sfP}$ is the fixed point of the Bellman
  operator $\cT^\pi$:
  \[
    \cT^\pi V(s) = \E_{a \sim \pi(\cdot|s)} \bracket*{ r(s,a) + \gamma \E_{s'
        \sim \h \sfP(\cdot|s,a)}[V(s')] }.
  \]
  We show that if $V$ is $L$-Lipschitz, then $\cT^\pi V$ is
  $(L_r^s + L_\pi L_r^a + \gamma (L_P^s + L_\pi L_P^a) L)$-Lipschitz.  Let $s_1, s_2 \in \sS$ and define
  $g(s, a) = r(s,a) + \gamma \E_{s' \sim \h \sfP(\cdot|s,a)}[V(s')]$.
  We decompose the difference as
  \begin{align*}
    \abs{\cT^\pi V(s_1) - \cT^\pi V(s_2)}
    & = \abs*{ \E_{a \sim \pi(\cdot|s_1)} [g(s_1,a)] - \E_{a \sim \pi(\cdot|s_2)} [g(s_2,a)] } \\
    & \leq \underbrace{\abs*{ \E_{a \sim \pi(\cdot|s_1)} [g(s_1,a) - g(s_2,a)] }}_{\text{(I): same distribution, different inputs}}
      + \underbrace{\abs*{ \E_{a \sim \pi(\cdot|s_1)} [g(s_2,a)] - \E_{a \sim \pi(\cdot|s_2)} [g(s_2,a)] }}_{\text{(II): different distributions, same input}}.
  \end{align*}
  For any fixed $a$, the Lipschitz assumptions on $r$ and $\h \sfP$ in $s$ give
  \[
    \abs{g(s_1, a) - g(s_2, a)} \leq (L_r^s + \gamma L_P^s L)\, d_\sS(s_1, s_2),
  \]
  where we used Kantorovich-Rubinstein duality to bound
  $\abs{\E_{\h \sfP(\cdot|s_1,a)}[V] - \E_{\h \sfP(\cdot|s_2,a)}[V]} \leq L \, W_1(\h \sfP(\cdot|s_1,a), \h \sfP(\cdot|s_2,a)) \leq L \, L_P^s \, d_\sS(s_1, s_2)$.
  Thus, we have $\text{(I)} \leq (L_r^s + \gamma L_P^s L) \, d_\sS(s_1, s_2)$.

  The function $g(s_2, \cdot)$ is $(L_r^a + \gamma L_P^a L)$-Lipschitz in $a$,
  by the Lipschitz assumptions on $r$ and $\h \sfP$ in $a$.
  By the Kantorovich-Rubinstein duality applied to the distributions
  $\pi(\cdot|s_1)$ and $\pi(\cdot|s_2)$ on~$\sA$, we have
  \[
    \text{(II)}
    \leq (L_r^a + \gamma L_P^a L) \, W_1(\pi(\cdot|s_1), \pi(\cdot|s_2))
    \leq (L_r^a + \gamma L_P^a L) \, L_\pi \, d_\sS(s_1, s_2).
  \]
  Combining the bounds for Terms (I) and (II):
  \[
    \abs{\cT^\pi V(s_1) - \cT^\pi V(s_2)}
    \leq \bracket*{L_r^s + L_\pi L_r^a +
      \gamma (L_P^s + L_\pi L_P^a) L} \, d_\sS(s_1, s_2).
  \]
  The fixed point $V_\pi^{\h \sfP}$ must satisfy
  $L_{v} = L_r^s + L_\pi L_r^a + \gamma (L_P^s + L_\pi L_P^a) L_{v}$, which yields
  $L_v = \frac{L_r^s + L_\pi L_r^a}{1 - \gamma (L_P^s + L_\pi L_P^a)}$,
  provided $\gamma (L_P^s + L_\pi L_P^a) < 1$.
\end{proof}

\newpage
\section{Algorithm Implementation and Experimental Details}
\label{sec:implementation}

In this section, we detail the practical implementation of our minimax simulator
learning framework. We adopt the formulation from
Algorithm~\ref{alg:active-wasserstein} but introduce two crucial modifications
for stability and efficiency: (1) we enforce the Lipschitz constraint on the
critic using a gradient penalty as is standard in WGAN-GP (Wasserstein
Generative Adversarial Network with Gradient Penalty)
\citep{gulrajani2017improved}; and (2) we constrain the adversarial exploration
to task-relevant regions.

\subsection{Task-aware active sampling via constrained optimization}

While our theory defines a zero-sum game over the space of \emph{all} policies
$\Pi$, in practice we are primarily interested in the simulator's accuracy on
the manifold of states visited by competent agents. A pure minimax adversary
might focus on modeling complex but irrelevant dynamics (e.g., the chaotic
motion of a robot after it has already fallen/failed), which wastes sampling
budget on regions that do not contribute to task performance.

To formally address this, we reframe the sampling step. Instead of an
unconstrained maximization of model error, we formulate it as a constrained
optimization problem. We seek the policy $\pi$ that maximizes the estimator's
error (finding the model's blind spots) subject to the constraint that the
policy remains competent at the task, achieving a return above a threshold
$V_{\min}$.
\[
  \max_{\pi \in \Pi} \quad \E_{\pi} \bracket*{ \sum_{t=0}^H D(s_t, a_t, s_{t+1})
  } \quad \text{s.t.} \quad \E_{\pi} \bracket*{ \sum_{t=0}^H
    r_{\mathrm{task}}(s_t, a_t) } \ge V_{\min}.
\]
We solve this problem via Lagrangian relaxation. We introduce a Lagrange
multiplier $\lambda \geq 0$ (represented by $1/\alpha$ in our algorithm) to
convert the hard constraint into a soft unconstrained objective:
\[
  \max_{\pi} \quad \E_{\pi} \bracket*{ \sum_{t=0}^H \paren*{ D(s_t, a_t,
      s_{t+1}) + \lambda r_{\mathrm{task}}(s_t, a_t) } } - \lambda V_{\min}.
\]
Since $\lambda V_{\min}$ is a constant with respect to $\pi$, maximizing this
Lagrangian is equivalent to maximizing the cumulative sum of a hybrid reward
function $r_{\mathrm{sample}}$:
\[
  r_{\mathrm{sample}}(s, a) = \alpha r_{\mathrm{task}}(s, a) + \bracket*{ D(s,
    a, s') - \E_{s'' \sim \h \sfP}[D(s, a, s'')] }.
\]
Here, the hyperparameter $\alpha$ controls the tightness of the constraint. When
$\alpha \to 0$, we recover the pure Minimax formulation (Global
Robustness). When $\alpha$ is large, the agent prioritizes task performance,
using the model error term purely as an intrinsic exploration bonus to explore
the boundaries of the relevant state space.

\subsection{Algorithm details and representation}

The full procedure is summarized in Algorithm~\ref{alg:implementation}. It
proceeds in three key iterative steps:

\begin{enumerate}

\item Metric Learning (Critic Step). We update the critic $D$ to maximize the
  Wasserstein distance between real transitions $(s, a, s')$ and \emph{dreamed}
  transitions $(s, a, s'')$.
  \begin{itemize}

  \item Representation: The critic $D_\phi(s, a, s')$ is a standard multi-layer
    perceptron neural network (MLP). Crucially, the final layer is a linear
    projection without activation (that is, no Sigmoid or Tanh). This is
    required because the Kantorovich-Rubinstein duality defines the distance as
    a supremum over 1-Lipschitz functions $f\colon \sS \to \Rset$, which must be
    unbounded. Sigmoid activations would restrict the range to $[0, 1]$,
    collapsing the metric to the Jensen-Shannon divergence.

  \item Gradient penalty: To enforce the Lipschitz constraint required for the
    linear objective to be valid, we apply a gradient penalty
    $\lambda_{\mathrm{GP}}$ to the norm of the critic's gradients.

  \end{itemize}

\item Simulator Learning (Model Step). The model $\h \sfP_\theta$ is updated to
  fool the critic.
  \begin{itemize}

  \item Representation: The model can be represented as a probabilistic neural
    network (e.g., an ensemble of Gaussian MLPs) that outputs the parameters of
    a distribution $s'' \sim \cN(\mu_\theta(s,a),
    \Sigma_\theta(s,a))$. Gradients are propagated through the sampling step
    using the reparameterization trick.

  \item Linear Objective: We update the model via gradient ascent on
    $\E[D(s, a, s'')]$. Unlike standard GANs which often require a log-loss
    surrogate (e.g., $\log D$) to prevent vanishing gradients when the
    discriminator saturates, our use of the Wasserstein metric with WGAN-GP
    ensures that the critic provides useful linear gradients even when the model
    is far from the target. Thus, we optimize the linear objective directly,
    preserving the theoretical connection to the $W_1$ distance.
  \item While we use the linear objective to guide the data collection, we also
    train a second model using log-loss on the same data and use this as the
    final simulator. Depending on the model architecture, this can help improve
    generalization. Importantly, unlike in GANs, the log-loss surrogate trained
    model is not affecting the metric learning and data collection.
  \end{itemize}

\item Active Sampling (Policy Step). The policy $\pi_\psi$ is updated to
  maximize the hybrid reward $r_{\mathrm{sample}}$ derived above. This guides
  data collection toward regions where the current model $\h \sfP_t$ has high
  discrepancy from the real dynamics, while staying on the manifold of relevant
  tasks.

\end{enumerate}

\textbf{Sample vs.\ Computational Efficiency.}
We acknowledge that Algorithm~\ref{alg:active-wasserstein} involves solving an
inner RL problem (optimizing $\pi$ in the Error-MDP) at each iteration. While
computationally more intensive than standard MLE or heuristic uncertainty
sampling (e.g., ensembles), this cost is justified in settings where
\emph{real-world samples} are the scarce resource (e.g., robotics, medical
AI). Our framework effectively trades increased offline computation for
minimized real-world interaction, ensuring that every sample collected targets a
strategically relevant deficiency in the model.
Furthermore, we note that finding the exact optimal policy $\pi^*$ for the
Error-MDP at every iteration is not strictly necessary. In practice, taking a
few gradient steps on the policy (e.g., via PPO) to improve its error-seeking
capability is sufficient to discover new, high-error regions, making the inner
loop computationally tractable.

The direct implementation of Algorithm~\ref{alg:active-wasserstein} samples from
the real environment iteratively, similar to standard MBRL methods. However, in
computationally constrained settings, one can decouple this by maintaining a
replay buffer and using importance sampling to sample heavily from the buffer's
historical data, falling back on the real environment only when the buffer lacks
coverage in the high-error regions identified by the critic.

\textbf{Inaccurate critic predictions harming training?}  Please note that
initial inaccuracies in the critic do not derail training because the algorithm
is inherently self-correcting. If the critic falsely assigns a high error score
to a specific state-action pair, the sampling distribution will place high
probability mass there. In the next iteration, the algorithm will collect
ground-truth samples from exactly this pair, automatically correcting both the
critic's estimate and the simulator's accuracy in that region.

\begin{algorithm}[H]
  \caption{Implementation of Algorithm~\ref{alg:active-wasserstein}
    (with WGAN-GP and Task-Aware Sampling)}
\label{alg:implementation}
\begin{algorithmic}[1]
  \STATE \textbf{Input:} Initial model $\h \sfP_0$ (Gaussian MLP),
  initial policy $\pi_0$, initial critic $D_0$ (Linear output),
  iterations $T$, GP coefficient $\lambda_{\mathrm{GP}} \ge 0$, task
  weight $\alpha \ge 0$.

\FOR{$t = 1$ to $T$}

\STATE \textbf{// 1. Data Collection.}

\STATE Sample trajectories with current policy $\pi_{t-1}$ to create
batch $B_t$ of transitions $(s, a, r_{\mathrm{task}}, s')$

\STATE \textbf{// 2. Critic update.}

\STATE Sample random interpolates
$\tilde{s} \leftarrow \e s' + (1 - \e) s''$ where
$s'' \sim \h \sfP_{t-1}(s, a)$ and $\e \sim U[0,1]$

\STATE Update critic $D_t$ by gradient ascent step on:
\[
  \E_{B_t} \bracket*{ \underbrace{D(s, a, s')}_{\text{Real}}
  - \E_{s''\sim \h \sfP_{t-1}}[D(s, a, s'')]
  - \lambda_{\mathrm{GP}} \underbrace{\paren*{
      \norm*{\nabla_{\tilde{s}} D(s, a, \tilde{s}) }_2 - 1}^2}_{\text{Gradient Penalty}} }
\]

\STATE \textbf{// 3. Model update.}
\STATE Update model $\h \sfP_t$ by gradient ascent step on:
\[
  \E_{(s, a)\sim B_t} \bracket*{ \E_{s''\sim \h \sfP(s,a)}\bracket*{ D_t(s, a, s'') } }
\]
\COMMENT{Maximizes the fake score.}

\STATE \textbf{// 4. Policy Update (task-aware).}
\STATE Compute hybrid reward for batch:
\[
  r_{\mathrm{sample}}(s, a) = \alpha r_{\mathrm{task}}
  + \bracket*{ D_t(s, a, s') - \E_{s'' \sim \h \sfP_t}[D_t(s, a, s'')] }
\]
\STATE Update policy $\pi_t$ (e.g., via SAC) to maximize
$\E_{\pi} [ r_{\mathrm{sample}}(s, a) ]$
\COMMENT{Constraint: Explores high-error regions within the task
  manifold.}

\ENDFOR

\STATE \textbf{Output:} Averaged model $\ov{\h \sfP}_T
= \frac{1}{T}\sum_{t=1}^T \h \sfP_t$
\end{algorithmic}
\end{algorithm}

\subsection{Additional Details on Experimental Setup}

We build our implementation for the experiments on policy training in control
tasks on top of the Acme framework~\citep{hoffman2020acme}. The implementation
of domains are from Gymnasium~\citep{towers2024gymnasium} and the DM-Control
Suite~\citep{tunyasuvunakool2020}. For ease of computation, we limit the horizon
for each task to 200 steps.

For policy, critic and simulator model, we use a feed forward neural network
with two hidden layers of 256 nodes each and ReLU activations.  The final layers
are all linear with the following output
\begin{itemize}
\item \textbf{Policy:} Mean and log standard deviation of the Gaussian from
  which the action is drawn.
\item \textbf{Critic:} Single output as the critic score.
\item \textbf{Simulator model:} Number of state dimensions outputs predicting
  the next state and an additional output predicting reward.
\end{itemize}

We use an Adam optimizer with learning rate as $3e-4$ for each model. The
baseline only uses the task reward when optimizing the data collection policy
while Algorithm~\ref{alg:active-wasserstein} uses a uniform mixture of task
reward and critic difference score, strategically seeking high error regions
that are relevant for high reward policies.

All experiments were performed on an internal cluster with $64$ CPUs per worker
and a H100 GPU. However, none of the neural networks and domains in our
empirical evaluation require extensive compute and we believe that the results
can be reproduced on a single consumer hardware machine with any modern GPU
within as most 1 week of compute time.

\paragraph{Details on Narrow Passage Domain}
We define a continuous 2D domain where relevant model errors are concentrated in
a small region that is rarely visited by uniform exploration, but strictly
necessary for optimal policy performance:

\begin{itemize}

\item State Space: $\sS = [0, 1]^2$. Agent starts at $(0.1, 0.5)$. Goal region
  is $x > 0.9$.

\item Action Space: $\sA = [-0.1, 0.1]^2$ (velocity vector $(\dot x, \dot y)$).

\item True Dynamics ($\sfM$):

  \begin{itemize}

  \item \emph{Normal Region:} $s_{t+1} = s_t + a_t + \mathcal{N}(0, 0.01I)$.

  \item \emph{The Wall:} A barrier at $x=0.5$ blocks all movement, \emph{except}
    for a narrow passage at $y \in [0.45, 0.55]$.

  \item \emph{The Trap:} Inside the narrow passage, there is a strong "wind"
    pushing the agent down: $\dot y_{\text{eff}} = \dot y - 0.05$.

  \end{itemize}

\item Initial Simulator ($\h \sfM_0$): A naive linear model,
  $s_{t+1} = s_t + a_t$, trained on initial uniform random data. It likely
  learns free movement everywhere, missing the wall and the trap.

\end{itemize}

While policies and models are stochastic (Gaussian) for experiments in control
tasks, we here use a deterministic transition model. Then the 1-Wasserstein
distance cleanly reduces to the expected distance between predicted and observed
state.  This simplification was an intended benefit of the experimental
design. It allowed us to clearly visualize and validate that the critic was
accurately learning the Error-MDP reward function without the added confounding
variance of sampling from a stochastic learned model.

\end{document}